\documentclass{fairmeta}
\setcitestyle{sort}

\usepackage{amsmath,amssymb,mathtools,amsthm}

\usepackage{array,makecell,adjustbox}
\usepackage[shortlabels,inline]{enumitem}
\usepackage{xspace}

\titlespacing*{\paragraph}{0pt}{1.25ex plus 1ex minus .2ex}{0.75em}
\usepackage{hyperref}       
\usepackage{url}            
\usepackage{booktabs}       
\usepackage{amsfonts}       
\usepackage{nicefrac}       
\usepackage{microtype}      
\usepackage{xcolor}         

\newtheorem{replayanalysisprop}{Proposition}[section]
\usepackage{graphicx}
\usepackage{caption}
\usepackage{subcaption}
\usepackage{array,tabularx} 
\usepackage{multirow}         
\usepackage{threeparttable}   
\usepackage{bm}
\usepackage{algorithm}        
\usepackage{algpseudocode}    
\usepackage{tikz}

\definecolor{lightgreen}{RGB}{200, 230, 200}
\definecolor{lightpink}{RGB}{249, 155, 157}
\definecolor{lightorange}{RGB}{248, 208, 151}
\definecolor{lightblue}{RGB}{200, 220, 255}
\definecolor{mediumseagreen}{RGB}{60, 179, 113}
\definecolor{steelblue}{RGB}{197, 229, 251}
\newcommand{\extra}[1]{\colorbox{lightpink}{\strut#1}}
\newcommand{\miss}[1]{\colorbox{lightorange}{\strut#1}}
\newcommand{\wrong}[1]{\colorbox{lightblue}{\strut#1}}
\theoremstyle{plain}

\theoremstyle{definition}

\theoremstyle{remark}

\title{ER-JEPA: Experience Replay Improves Joint-Embedding Predictive Learning in Language Models}

\author[1]{Jingnan Pu}
\author[1]{Zi-En Fan}
\author[1,*]{Feng Lian}

\affiliation[1]{School of Automation Science and Technology, Xi'an Jiaotong University\\
No.~28, West Xianning Road, Xi'an, Shaanxi 710049, China}
\contribution[*]{Corresponding author}

\abstract{
Large language models (LLMs) excel at token-level generation but may learn undesirable abstract semantics and lack comprehensive perception. LLM-JEPA mitigates this by aligning different views of the same underlying knowledge via a joint-embedding predictive architecture (JEPA). However, strong alignment does not necessarily lead to accurate, stable predictions. To address this, we propose ER-JEPA, which adds an episodic replay path to LLM-JEPA. ER-JEPA stores training pairs in a memory. At each step, it stores and retrieves relevant data to provide additional supervision. This enables learning from both the current batch and stored training pairs, providing additional supervision for token prediction and representation alignment. Experiments across multiple datasets (NL-RX, GSM8K, Spider, and NQ-Open) demonstrate that ER-JEPA consistently outperforms LLM-JEPA.

}

\begin{document}
\maketitle

\begin{figure}[htbp]
	\centering
	\includegraphics[width=\linewidth]{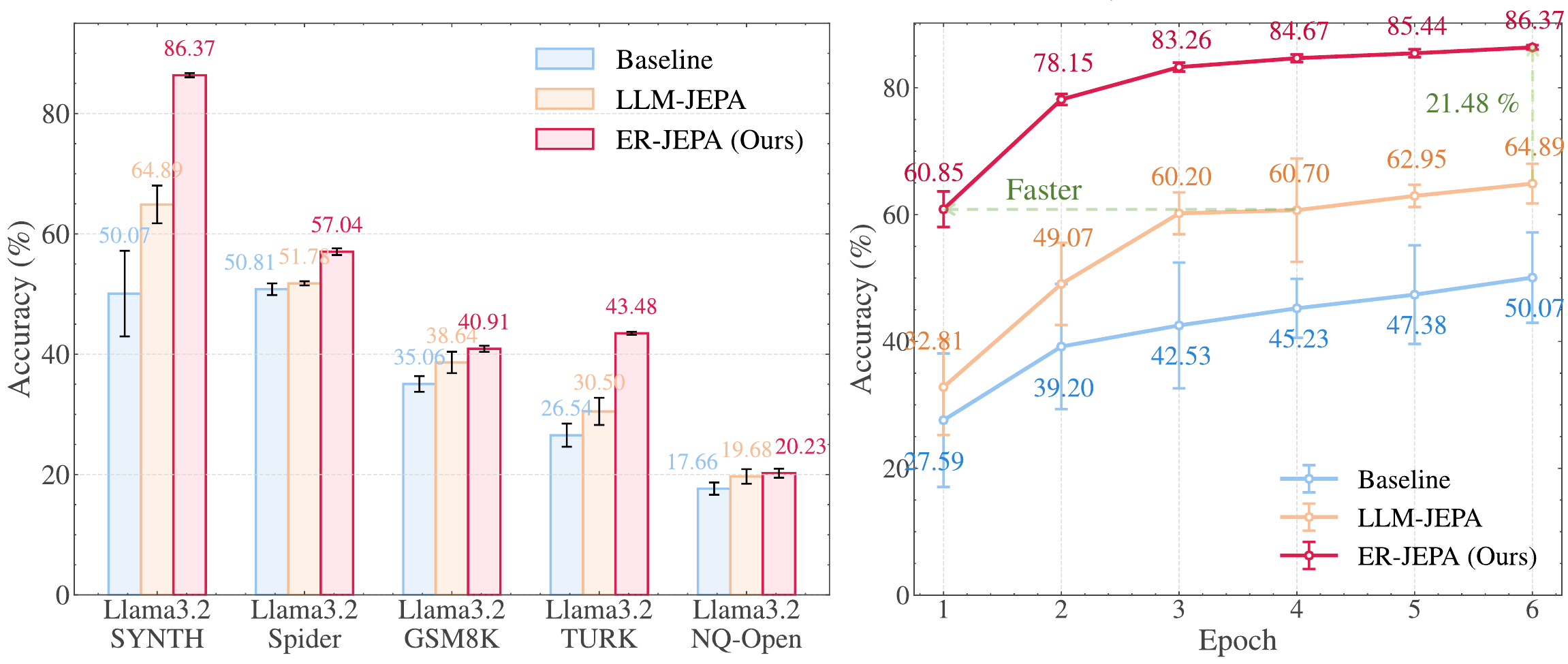}
	\caption{\textbf{ER-JEPA outperforms LLM-JEPA across tasks.}}
	\label{fig1:main_results}
\end{figure}


\section{Introduction}
Large language models (LLMs) have achieved remarkable advances in fluent generation \citep{pillutla2021mauve,yu2023megabyte,pmlr-v267-li25ch,gong2024evaluation,Li2024FromCD,liu2024deepseek}, few-shot learning \citep{brown2020language,wei2021finetuned,wang2024mixture}, and broad downstream transfer \citep{rafailov2023direct,zeng2022glm,Huan2025DoesMR}. These abilities require models to extract shared structure from large-scale data and convert it into dependable predictions on new problems. Autoregressive pre-training addresses both needs through next-token prediction (NTP) (see Figure~\ref{fig:ER_JEPA_Comparison_Redrawn}(a)), which provides a strong learning signal and can lead to semantic representations under the right conditions \citep{jin2024emergent}. However, a token-level objective may also favor local coherence over the representations needed for comprehensive perception and reasoning.

Joint-Embedding Predictive Architectures (JEPAs) shift prediction from the token level to the semantic level. Recent work on JEPAs has shown promising advances in representation learning and provable benefits for perception tasks \citep{assran2023self,bardes2024revisiting,assran2025v,chen2026vljepa}. Inspired by JEPAs, Huang et al. introduced LLM-JEPA \citep{huang2026llmjepa}, which adds an embedding space predictive loss to the standard language modeling loss (see Figure~\ref{fig:ER_JEPA_Comparison_Redrawn}(b)). This loss aligns different views of the same underlying knowledge while preserving generative capability. It predicts both the next token and the target representation, improving the perception capability of LLMs.

Despite these advances, better representation alignment does not by itself guarantee reliable predictions. We observe that LLM-JEPA continues to make errors even after its alignment loss has converged (Figure~\ref{fig:alignment_vs_answers}). Some previously correct predictions also become incorrect later in training. Therefore, progress in representation learning is not automatically converted into predictions that are learned and retained. This leaves an open question: how can LLM-JEPA correct persistent errors and retain correct predictions during training?

We explore experience replay as a way to address this question by revisiting previously seen training examples and providing additional supervision as the model changes. Like standard fine-tuning, LLM-JEPA supervises an example only when it appears in the current mini-batch. Between two visits, many updates change the model, and the example is not checked again until the data schedule reaches it in the next epoch. Replay can reduce this gap by revisiting past examples during training.
The replay mechanism is widely used in continual learning to preserve learned knowledge \citep{shi2024continuallearninglargelanguage,huang2024mitigatingcatastrophicforgettinglarge,deng2025unlockingpowerrehearsalcontinual}, and Complementary Learning Systems (CLS) theory gives it a similar role in the brain \citep{kumaran2016learning}. In CLS, a fast system stores individual experiences and replays them to a slow system that extracts shared structure. From this perspective, LLM-JEPA is similar to the former but has no counterpart to the fast system. This motivates our research question: \textbf{can experience replay improve the predictive performance of LLM-JEPA?}

To answer it, we propose ER-JEPA, which adds an episodic replay (ER) pathway that stores past training examples and replays them during later updates (see Figure~\ref{fig:ER_JEPA_Comparison_Redrawn}(c)). It enables the model to learn from useful past experiences beyond the current mini-batch, providing additional supervision signals and strengthening the perception capabilities of LLMs. The replay path is removed after training, so ER-JEPA has the same architecture and inference cost as LLM-JEPA.

In summary, the contributions of this paper are as follows:

\begin{itemize}
	\item We propose ER-JEPA, which extends LLM-JEPA with an episodic replay path that revisits past training examples to provide additional supervision for token prediction and representation alignment. We show that learning from past samples continues to provide alignment updates and reinforces the knowledge important for perception and reasoning.

	\item  We implement content-based, uniform, and hard selection within a shared replay path. We demonstrate that the replay branch itself, rather than any specific retrieval rule, is the central source of improvement.
	
	\item We evaluate ER-JEPA across multiple datasets. The results show that ER-JEPA consistently outperforms LLM-JEPA, demonstrating the benefit of episodic replay for abstraction and generalization.
	
\end{itemize}

\begin{figure*}[htbp]
	\centering
	\includegraphics[width=\linewidth]{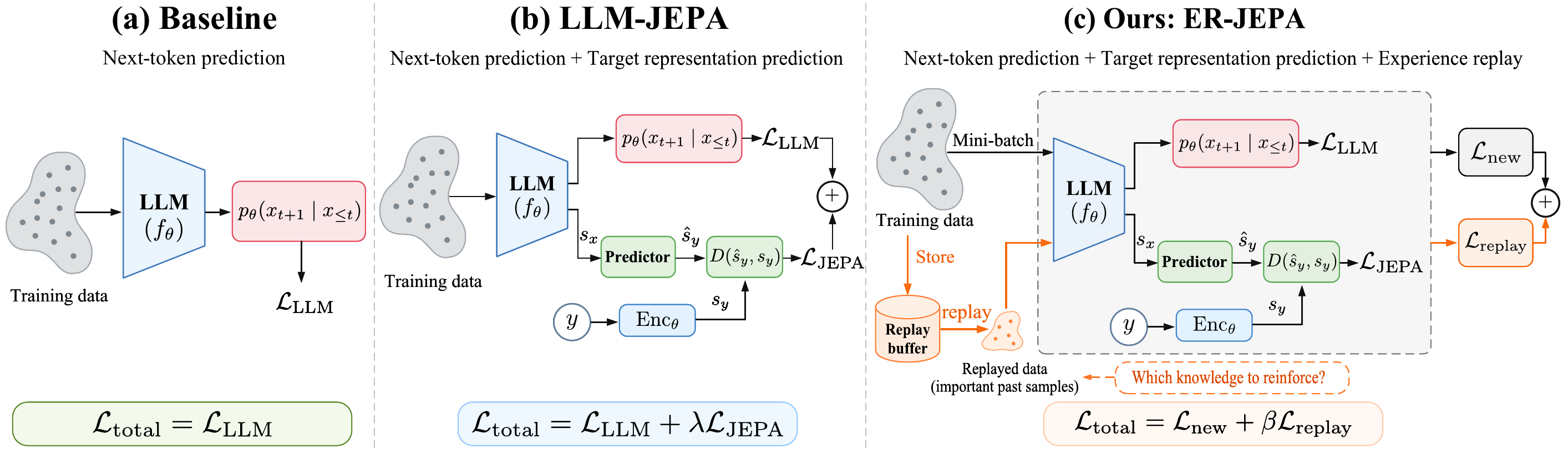}
	\caption{\textbf{Comparison of baseline training, LLM-JEPA, and ER-JEPA.} (a) The baseline learns through next token prediction. (b) LLM-JEPA adds prediction in the embedding space to learn abstract relations between inputs and targets. (c) ER-JEPA adds a replay path to LLM-JEPA. It stores past training examples and replays their tokens, providing additional supervision for both token prediction and representation alignment. Past experience thus continues to shape learning as the model changes, helping it correct errors and reinforce learned knowledge for more reliable predictions.}
	\label{fig:ER_JEPA_Comparison_Redrawn}
\end{figure*}

\section{Methodology}
\label{sec:method:replay_first}

ER-JEPA extends LLM-JEPA with a replay path over stored training examples. We first introduce the joint training objective in Section~\ref{sec:method:replay_first:objective}. We then explain how replay provides additional supervision in Section~\ref{sec:method:replay_first:training}.

\subsection{ER-JEPA Overview and Objective}
\label{sec:method:replay_first:objective}

Our ER-JEPA builds on LLM-JEPA, retaining its token prediction and representation alignment objectives while adding an episodic replay path to learn from past training examples. The original LLM-JEPA objective is
\begin{equation}
	\mathcal{L}_{\rm LLM\text{-}JEPA}
	=
	\underbrace{
		\sum_{\ell=2}^{L}
		\mathcal{L}_{\rm LLM}\!\left(\mathrm{Text}_{1:\ell-1},\mathrm{Text}_{\ell}\right)
	}_{\text{generative capabilities (LLM)}}
	+
	\lambda\,
	\underbrace{
		d\!\left(
		{\rm Pred}_{\phi}({\rm Enc}_{\theta}(\mathrm{Text})),\;
		{\rm Enc}_{\theta}(\mathrm{Code})
		\right)
	}_{\text{abstraction capabilities (JEPA)}}.
	\label{eq:replay_first:base_objective}
\end{equation}
Here $\mathrm{Text}$ and $\mathrm{Code}$ denote the source and target views, and $L$ is the source sequence length. $\mathrm{Enc}_{\theta}$ and $\mathrm{Pred}_{\phi}$ denote the encoder and predictor. The function $d$ measures cosine distance, and $\lambda$ controls the JEPA loss weight.

For a fresh mini-batch $\mathcal{B}_t$ of size $B$, let $x_i^u$ and $x_i^a$ denote the source and target tokens of example $i$, and $x_i^{\rm full}=x_i^u\Vert x_i^a$. Define
\begin{equation}
	z_i^u = \mathrm{Enc}_\theta(x_i^u),
	\qquad
	p_i = \mathrm{Pred}_\phi(z_i^u),
	\qquad
	z_i^a = \mathrm{Enc}_\theta(x_i^a).
	\label{eq:replay_first:new_representations}
\end{equation}
Let $\mathcal{T}_i^{\rm new}$ contain the target-token positions in $x_i^{\rm full}$. The loss on new data is
\begin{equation}
	\mathcal{L}_{\rm LLM\text{-}JEPA\text{-}new}(\mathcal{B}_t)
	=
	\frac{\gamma}{Z_t^{\rm new}}
	\sum_{i=1}^{B}\sum_{\ell\in\mathcal{T}_i^{\rm new}}
	\mathcal{L}_{\rm LLM}\!\left(x_{i,<\ell}^{\rm full},x_{i,\ell}^{\rm full}\right)
	+\frac{\lambda}{B}\sum_{i=1}^{B}d(p_i,z_i^a),
	\label{eq:replay_first:new_loss}
\end{equation}
where $Z_t^{\rm new}=\sum_{i=1}^{B}|\mathcal{T}_i^{\rm new}|$ and $\gamma=1$ is the language-modeling loss weight.

The full ER-JEPA objective adds a replay term,
\begin{equation}
	\mathcal{L}_{\rm ER\text{-}JEPA}(\mathcal{B}_t)
	=
	\mathcal{L}_{\rm LLM\text{-}JEPA\text{-}new}(\mathcal{B}_t)
	+
	\beta\,\mathcal{L}_{\rm LLM\text{-}JEPA\text{-}replay}(\mathcal{B}_t),
	\label{eq:replay_first:total}
\end{equation}
where $\beta\geq 0$ controls the strength of replay. $\mathcal{L}_{\rm LLM\text{-}JEPA\text{-}replay}$ is defined in Section~\ref{sec:method:replay_first:training}.

\subsection{Replay-Based Training}
\label{sec:method:replay_first:training}

Content-based, uniform, and hard replay are three implementations of the same replay path. They share the replay objective below. The selection rules for all three strategies are given in Appendix~\ref{sec:appendix:replay_policies}. Unless otherwise specified, ER-JEPA uses content-based replay.

At each step, a replay strategy selects up to $R$ entries from the episodic memory, where $R$ is the replay budget. Let $R_t$ be the number of selected entries and $(\widetilde{\tau}_i^u,\widetilde{\tau}_i^a)$ be the stored source and target tokens for entry $i$. Memory storage and retrieval are described in Appendix~\ref{sec:method:replay_first:memory}.

For each replay entry $i$, we concatenate the stored source and target tokens as $\widetilde{\tau}_i^{\rm full}=\widetilde{\tau}_i^u\Vert\widetilde{\tau}_i^a$. The current model recomputes the source and target representations,
\begin{equation}
	z_i^{u,{\rm rep}}=\mathrm{Enc}_{\theta}(\widetilde{\tau}_i^u),
	\qquad
	z_i^{a,{\rm rep}}=\mathrm{Enc}_{\theta}(\widetilde{\tau}_i^a),
	\qquad
	p_i^{\rm rep}=\mathrm{Pred}_{\phi}(z_i^{u,{\rm rep}}).
	\label{eq:replay_first:recomputed_representations}
\end{equation}

Let $\mathcal{T}_i^{\rm rep}$ contain the target-token positions in $\widetilde{\tau}_i^{\rm full}$. For $R_t>0$, the replay loss is
\begin{equation}
	\mathcal{L}_{\rm LLM\text{-}JEPA\text{-}replay}(\mathcal{B}_t)
	=
	\frac{\gamma}{Z_t^{\rm rep}}
	\sum_{i=1}^{R_t}\sum_{\ell\in\mathcal{T}_i^{\rm rep}}
	\mathcal{L}_{\rm LLM}\!\left(\widetilde{\tau}_{i,<\ell}^{\rm full},\widetilde{\tau}_{i,\ell}^{\rm full}\right)
	+
	\frac{\lambda}{R_t}
	\sum_{i=1}^{R_t}d\!\left(p_i^{\rm rep},z_i^{a,{\rm rep}}\right),
	\label{eq:replay_first:token_replay_objective}
\end{equation}

where $Z_t^{\rm rep}=\sum_{i=1}^{R_t}|\mathcal{T}_i^{\rm rep}|$. The loss is zero when $R_t=0$. The normalization matches the new-data loss in Eq.~\eqref{eq:replay_first:new_loss}.

The current and replay losses jointly update the shared model through Eq.~\eqref{eq:replay_first:total}. Replay-time JEPA errors refresh the selected memory scores, and replay counts are incremented as specified in Appendix~\ref{sec:appendix:replay_memory}. At inference time, the episodic pathway is removed.




\begin{figure}[htbp]
	\centering
	\includegraphics[width=\linewidth]{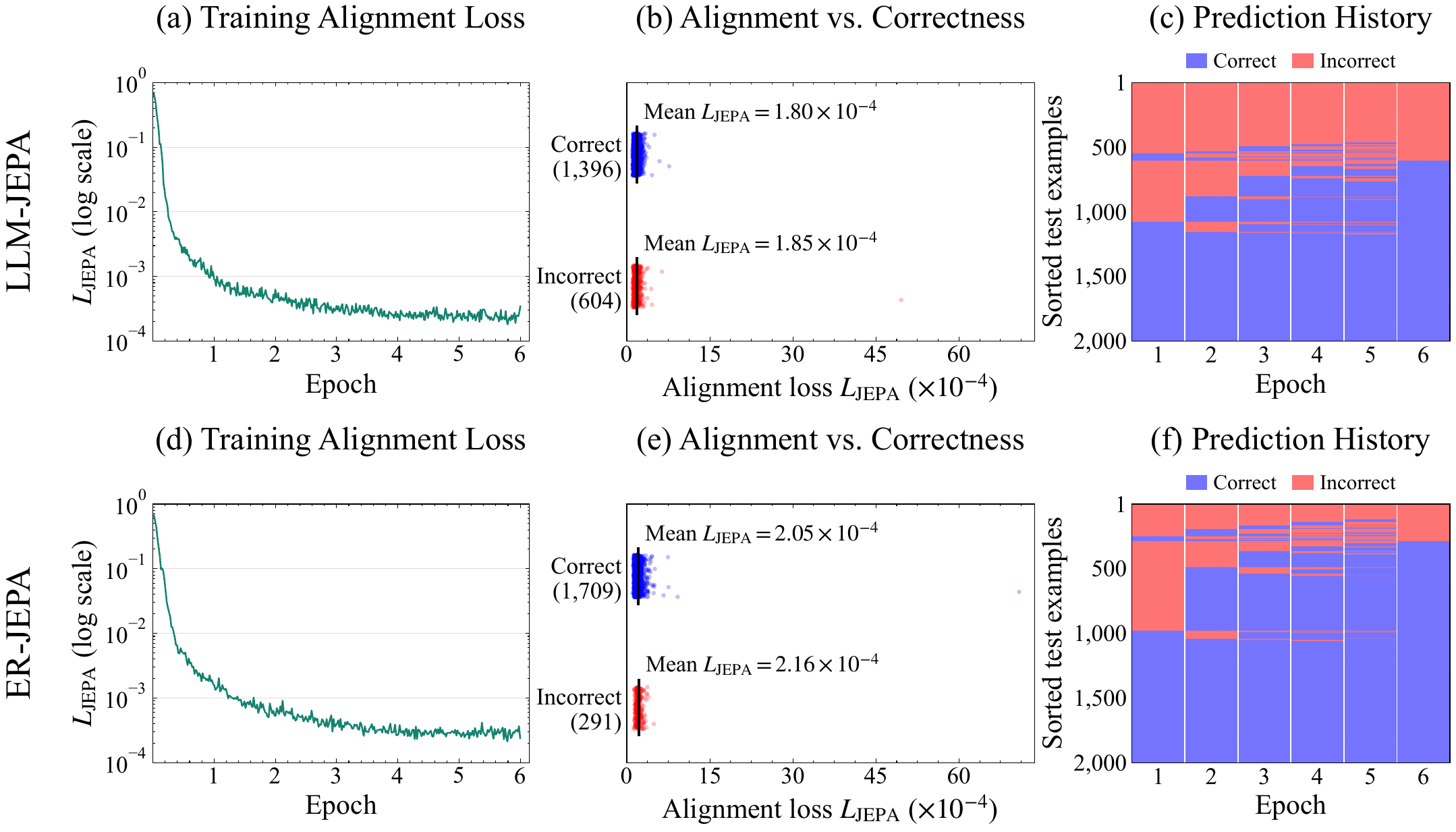}
	\caption{\textbf{LLM-JEPA can achieve strong alignment yet still make errors and lose previously correct predictions.} For LLM-JEPA, the JEPA loss converges in (a). Yet in (b) and (c), errors remain even when cosine similarities between inputs and targets are close to one. Some correct predictions also become incorrect as training continues. For ER-JEPA, the JEPA loss also converges in (d). In (e) and (f), more predictions are correct despite slightly lower cosine similarities. At each checkpoint, ER-JEPA consistently has more correct answers than LLM-JEPA. These results suggest that replay supervision helps resolve errors that alignment alone leaves unresolved. This figure uses Llama-3.2-1B on 2000 SYNTH test examples with seed 82. Heatmap rows are sorted independently within each method by final-checkpoint correctness and prediction history.}
	\label{fig:alignment_vs_answers}
\end{figure}
\section{Experiments: Episodic Replay Improves LLM-JEPA}
\label{sec:experiments}
This section evaluates whether episodic replay improves JEPA-based LLM training. After describing the experimental setup in Section~\ref{sec:experiments:setup}, we report results across tasks and compare accuracy at matched training compute in Section~\ref{sec:experiments:main}. We then examine the role of historical replay in Section~\ref{sec:experiments:compute_and_mechanism}, analyze prediction quality across regex structures and lengths  and report resource costs in Section~\ref{sec:experiments:prediction_quality}. Further ablations are provided in Appendix~\ref{app:ablation}.

\subsection{Experimental Setup}
\label{sec:experiments:setup}

We evaluate ER-JEPA under the same task families used in the LLM-JEPA protocol. Experiments are conducted on Llama-3.2-1B-Instruct~\citep{llama3}. We use five datasets with naturally paired input--target structures: NL-RX-SYNTH and NL-RX-TURK~\citep{locascio2016neural} for generating regular expressions from descriptions in natural language, GSM8K~\citep{cobbe2021training} for mathematical reasoning, Spider~\citep{yu2018spider} for text-to-SQL generation, and NQ-Open~\citep{lee-etal-2019-latent-no-open} for open-domain question answering. Following the evaluation protocol of LLM-JEPA, we report exact match accuracy of the generated regular expression for NL-RX-SYNTH and NL-RX-TURK, exact match accuracy of the final answer for GSM8K and NQ-Open, and execution accuracy for Spider. Detailed experimental settings are provided in Appendix~\ref{sec:appendix:experimental_setup}.

\subsection{ER-JEPA Outperforms LLM-JEPA}
\label{sec:experiments:main}

\paragraph{Episodic replay improves LLM-JEPA across tasks.}
\label{subsec:main_results}

We first test whether replay provides benefits beyond representation alignment alone. We fine-tune Llama-3.2-1B-Instruct with different training objectives on all five datasets described above. As shown in Figure~\ref{fig1:main_results}~(left), ER-JEPA achieves higher performance than LLM-JEPA on all five datasets. 

The consistent gains show that the benefit of ER-JEPA is not limited to one type of task. The method improves both structured generation tasks, such as regular expression and SQL generation, and more general generation tasks, such as mathematical reasoning and open-domain question answering. These results suggest that episodic replay improves the use of training data in LLM-JEPA. By storing separated traces of past examples and retrieving relevant traces for the current context, ER-JEPA allows the model to reuse useful past experiences during optimization. This mechanism provides additional training signals beyond the current mini-batch and leads to more effective fine-tuning.

\paragraph{ER-JEPA achieves higher accuracy at the same training compute.}
To test whether the gains persist at the same training compute, we evaluate all methods at the same budgets from $40.05$ to $240.31$ PFLOPs on SYNTH with Llama-3.2-1B. We also evaluate three replay policies, including content, uniform, and hard replay, to test whether the benefit depends on the replay selection rule. The three policies share the same replay path and differ only in how stored examples are selected.

Figure~\ref{fig:accuracy_vs_compute}(a) shows that all three replay policies outperform LLM-JEPA at every evaluated compute budget. At $240.31$ PFLOPs, content replay reaches $83.65 \pm 1.48\%$, while LLM-JEPA reaches $68.75 \pm 4.92\%$, a gain of $14.90$ percentage points. The normalized AUC also increases from $56.61 \pm 3.41\%$ to $73.73 \pm 1.15\%$. The gains across all three rules show that the benefit is not limited to content-based selection. At matched training compute, ER-JEPA achieves higher mean accuracy than LLM-JEPA on SYNTH with all three replay policies.
\begin{figure*}[htbp]
	\centering
	\begin{subfigure}[b]{0.48\linewidth}
		\centering
		\includegraphics[width=\linewidth]{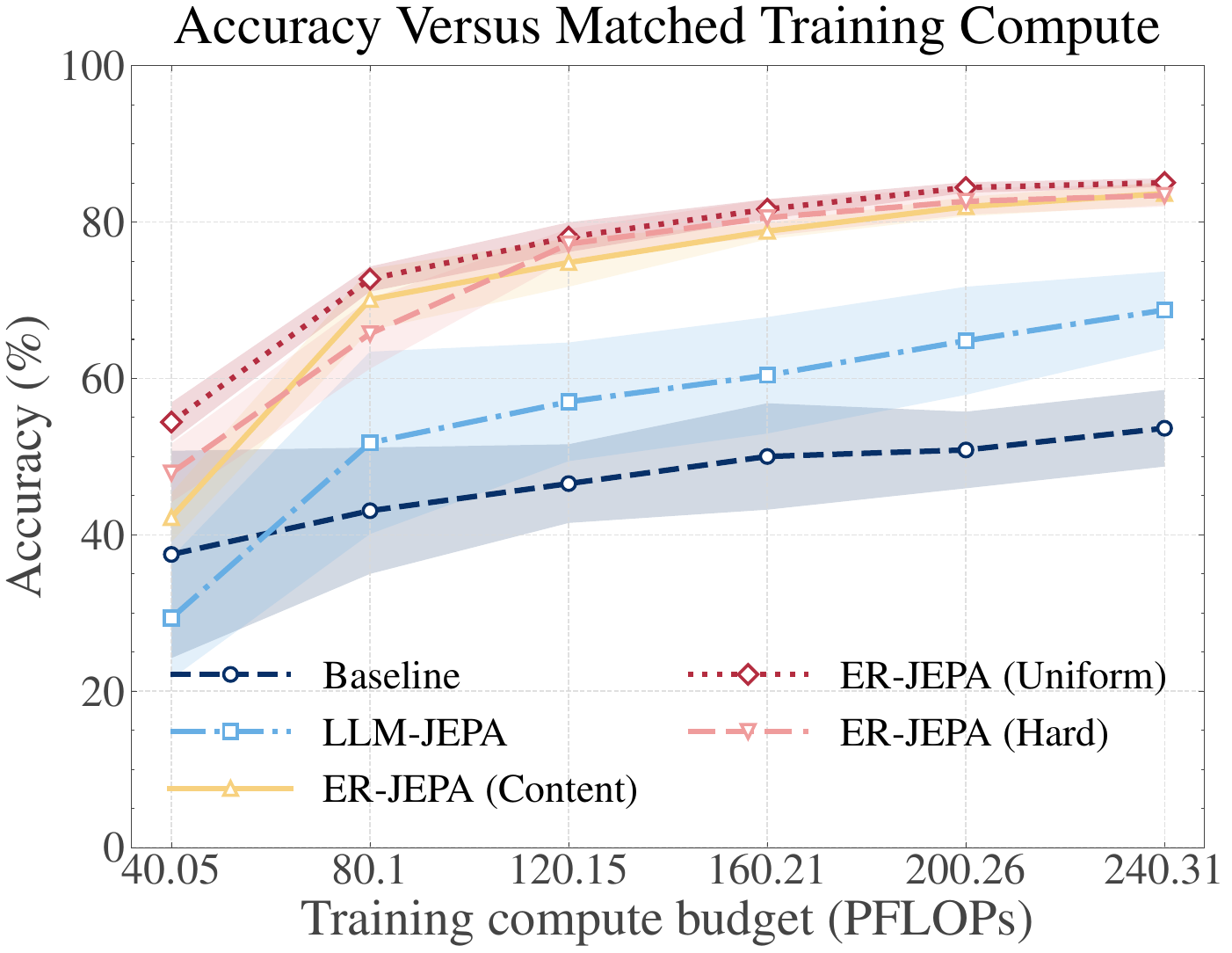}
		\caption{Five seeds.}
		\label{fig:accuracy_vs_compute_five_seeds}
	\end{subfigure}\hfill
	\begin{subfigure}[b]{0.48\linewidth}
		\centering
		\includegraphics[width=\linewidth]{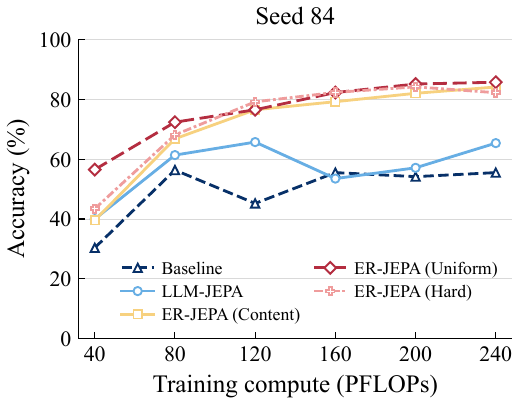}
		\caption{Seed $=84$.}
		\label{fig:accuracy_vs_compute_seed84}
	\end{subfigure}
	\caption{\textbf{ER-JEPA improves accuracy at matched training compute and reduces overfitting.} (a) All three replay policies outperform LLM-JEPA at every evaluated budget. These gains support the benefit of introducing the replay path. (b) Baseline loses accuracy after an early increase. LLM-JEPA delays this decline but still loses accuracy later. Content and uniform replay maintain their gains, while hard replay shows only a small final decline. These trends suggest that replay mitigates overfitting. They complement the mean curves in Figure~\ref{fig1:main_results} (right). Results use Llama-3.2-1B on SYNTH at six compute budgets from $40.05$ to $240.31$ PFLOPs.}
	\label{fig:accuracy_vs_compute}
\end{figure*}

\paragraph{Replay reduces accuracy drops during training.}
The mean curves in Figure~\ref{fig:accuracy_vs_compute}(a) rise at every budget for all methods, but individual runs behave differently.
Baseline and LLM-JEPA often lose accuracy that they reached earlier in training. These drops happen at different budgets in different seeds, so averaging hides them. Figure~\ref{fig:accuracy_vs_compute}(b) illustrates this pattern for seed 84. Baseline loses accuracy after an initial increase, and LLM-JEPA shows a decline at a later checkpoint. ER-JEPA shows strong resistance to overfitting under all three replay policies. These results suggest that the replay path mitigates overfitting and helps retain the accuracy gained earlier in training. Appendix Figure~\ref{fig:appendix_single_seed_curves} shows the other four seeds.

\subsection{Understanding the Role of Historical Replay}
\label{sec:experiments:compute_and_mechanism}

To analyze these gains, we first compare ER-JEPA with token-matched LLM-JEPA to assess whether additional token exposure alone can reproduce the improvement. We then compare historical replay with a control that applies the same replay loss to current-batch examples, assessing the additional benefit of revisiting past samples. We also examine which errors replay corrects and whether it keeps correct predictions. Appendix~\ref{app:replay_supervision_analysis} analyzes a single replay update.

\paragraph{Content replay achieves higher mean accuracy than token-matched LLM-JEPA.}
We first test whether extra token exposure can explain the gain. The token-matched control uses the token budget of content replay and retains the LLM-JEPA objective. Figure~\ref{fig:mechanism_token_exposure} shows that content replay reaches $83.65 \pm 1.48\%$, compared with $62.25 \pm 20.80\%$ for token-matched LLM-JEPA. These results indicate that extra token exposure alone does not reproduce the gain.

\paragraph{Historical replay achieves higher mean accuracy than the current-batch control.}
We next examine whether historical samples provide value beyond an additional loss on the current batch. The current-batch control applies the replay loss to current-batch examples instead of stored examples. Figure~\ref{fig:mechanism_historical_replay} shows that all three historical replay policies have higher mean accuracy than this control. Uniform replay reaches $85.04 \pm 0.54\%$, which is $3.80$ percentage points (pp) above the control, and it is higher in all five seeds. Content replay reaches $83.65\%$ ($+2.41$~pp), and hard replay reaches $83.41\%$ ($+2.17$~pp). These results are consistent with a benefit from reusing historical samples. 

\begin{figure*}[t]
	\centering
	\begin{minipage}[t]{0.48\textwidth}
		\vspace{0pt}
		\centering
		\includegraphics[width=\linewidth]{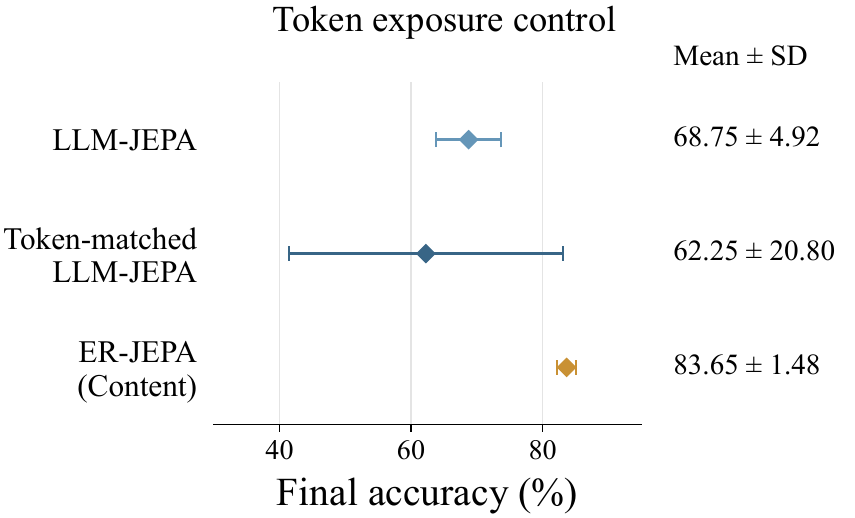}
		\caption{\textbf{ER-JEPA achieves higher mean accuracy with replay under matched token exposure.} Results use Llama-3.2-1B on SYNTH over five seeds.}
		\label{fig:mechanism_token_exposure}
	\end{minipage}\hfill
	\begin{minipage}[t]{0.48\textwidth}
		\vspace{0pt}
		\centering
		\includegraphics[width=\linewidth]{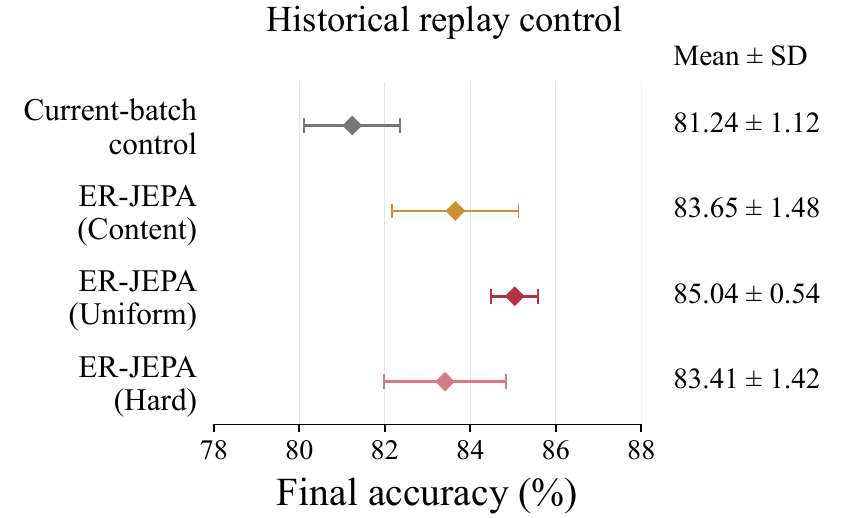}
		\caption{\textbf{Historical samples provide additional value.} Results use Llama-3.2-1B on SYNTH over five seeds.}
		\label{fig:mechanism_historical_replay}
	\end{minipage}
\end{figure*}

\paragraph{ER-JEPA achieves higher correction rates on examples mispredicted by the baseline.}
As discussed in the Introduction, a model can align well with its target representation and still predict the wrong answer, sometimes even losing ground it had already covered. This raises a natural question: \textbf{when Baseline errs, can LLM-JEPA and ER-JEPA fix its mistakes? }To find out, we test both models on the SYNTH inputs where Baseline fails at the final matched compute point, checking whether each model recovers the exact target expression. We sort these failures by length relative to the target, including over-generation (extra regex tokens), under-generation (missing tokens), and same-length mismatch (correct length, wrong expression). Table~\ref{tab:baseline_error_correction} shows that ER-JEPA achieves a higher mean correction rate in all three categories. The largest gain occurs for over-generation, where the rate rises from $53.07\%$ with LLM-JEPA to $84.51\%$ with ER-JEPA. Complete correction counts and rates for each seed are reported in Appendix Table~\ref{tab:baseline_error_correction_by_seed}. Appendix Figure~\ref{fig:llm_jepa_error_outcomes} further compares the two methods on inputs that LLM-JEPA predicts incorrectly. 

To further examine whether replay helps correct errors and retain correct predictions during training, we track predictions on the same test inputs across successive checkpoints. Figure~\ref{fig:alignment_vs_answers} shows more correct predictions for ER-JEPA in (f) than for LLM-JEPA in (c) at every evaluated checkpoint. Figure~\ref{fig:prediction_dynamics}(b) shows a higher mean correction rate for ER-JEPA than for LLM-JEPA. In Figure~\ref{fig:prediction_dynamics}(c), the mean rate of correct predictions becoming incorrect decreases from $8.85\%$ with LLM-JEPA to $6.65\%$ with ER-JEPA. These results suggest that replay supports both error correction and the retention of correct predictions. Example inputs, targets, and predictions from the three methods are shown in Appendix Figure~\ref{fig:appendix_rule_selected_examples}.

\begin{table}[t]
	\centering
	\caption{\textbf{ER-JEPA achieves higher mean correction rates across all three Baseline error categories.} Results use Llama-3.2-1B on SYNTH at 240.309 PFLOPs. Values are means $\pm$ one sample standard deviation over five seeds. Correction rates are computed separately for each seed before averaging.}
	\label{tab:baseline_error_correction}
	\setlength{\tabcolsep}{3pt}
	\renewcommand{\arraystretch}{1.15}
	\begin{tabular}{@{}lclcc@{}}
        \toprule
        Baseline error category & Baseline errors & Method & Corrected & Correction rate (\%) \\
        \midrule
        \multirow{2}{*}{Over-generation} & \multirow{2}{*}{$718.4 \pm 112.5$}
        & LLM-JEPA & $378.8 \pm 64.6$ & $53.07 \pm 7.92$ \\
        & & ER-JEPA & $610.8 \pm 126.1$ & $84.51 \pm 4.70$ \\
        \multirow{2}{*}{Under-generation} & \multirow{2}{*}{$2.6 \pm 1.5$}
        & LLM-JEPA & $0.6 \pm 0.5$ & $24.00 \pm 25.10$ \\
        & & ER-JEPA & $1.2 \pm 0.8$ & $54.67 \pm 44.07$ \\
        \multirow{2}{*}{\shortstack[l]{Same-length\\mismatch}} & \multirow{2}{*}{$168.4 \pm 10.9$}
        & LLM-JEPA & $34.8 \pm 5.8$ & $20.58 \pm 2.35$ \\
        & & ER-JEPA & $42.4 \pm 7.0$ & $25.07 \pm 2.62$ \\
        \bottomrule
    \end{tabular}
\end{table}

\begin{figure}[htbp]
	\centering
	\includegraphics[width=\linewidth]{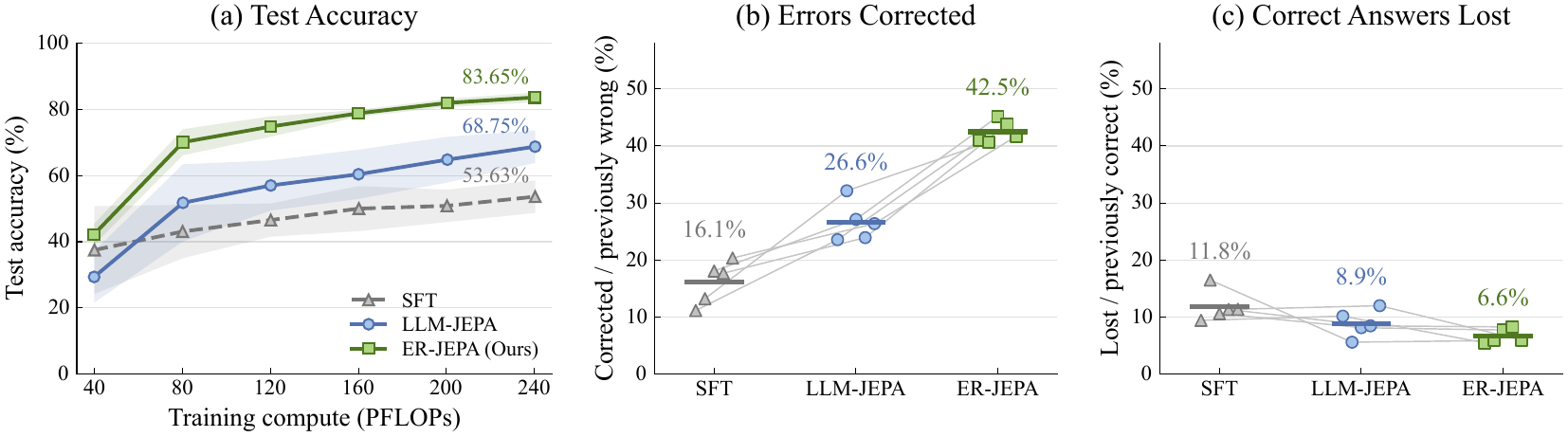}
	\caption{\textbf{ER-JEPA achieves higher accuracy and corrects more errors during training.} (a) Test accuracy at comparable training compute for SFT, LLM-JEPA, and ER-JEPA. (b) The fraction of wrong predictions that become correct at the next checkpoint. (c) The fraction of correct predictions that become wrong at the next checkpoint. Compared with LLM-JEPA, ER-JEPA improves error correction across all five seeds, while the reduction in lost correct predictions is smaller and varies across seeds. Results use Llama-3.2-1B on the same 2,000 SYNTH test examples.}
	\label{fig:prediction_dynamics}
\end{figure}
\subsection{Fine-grained Analysis of Prediction Quality}
\label{sec:experiments:prediction_quality}

We next examine where the accuracy gains occur within the SYNTH test set, focusing on the structure and length of the target regular expressions.

\paragraph{Accuracy gains span multiple forms of symbolic composition.}
To examine how these gains relate to symbolic structure, we analyze the SYNTH test set using six structural subsets defined by the target regular expressions, including complement, intersection, alternation, quantifiers, boundaries and anchors, and character classes. Detailed definitions and classification rules are provided in Appendix Table~\ref{tab:regex_structure_definitions}. Figure~\ref{fig:regex_structure_accuracy} shows that ER-JEPA achieves higher mean exact match accuracy than LLM-JEPA in all six subsets. The largest gains occur for expressions containing quantifiers, intersections, and character classes, with improvements of $16.90$, $16.07$, and $15.42$ percentage points, respectively. These structures encode repetition, conjunction, and restrictions on allowed characters, all of which must be preserved when translating a natural language description into a regular expression. The gains therefore span several forms of constraint composition. This pattern is consistent with replay reinforcing the mapping from linguistic constraints to their symbolic implementation.

\paragraph{Replay reduces sensitivity to expression length.}
We next examine whether this advantage persists for longer expressions. Target regex length serves as a proxy for complexity. We divide the SYNTH test set into four groups: $\leq 10$, $11$--$15$, $16$--$20$, and $>20$ regex tokens. Figure~\ref{fig:regex_length_accuracy} shows that Regular is highly sensitive to expression length. Its mean accuracy drops from $83.64\%$ in the shortest group to $33.53\%$ in the longest group, a decrease of $50.11$ percentage points. LLM-JEPA reduces this gap, but its accuracy still falls from $85.15\%$ to $56.76\%$. ER-JEPA achieves the highest accuracy in all four groups. Its accuracy decreases from $89.80\%$ to $78.53\%$, a drop of $11.27$ percentage points, compared with $28.39$ for LLM-JEPA. These results suggest that episodic replay helps the model generate longer symbolic expressions more reliably.

\begin{figure*}[htbp]
	\centering
	\begin{minipage}[t]{0.48\textwidth}
		\vspace{0pt}
		\centering
		\parbox[c][0.60\linewidth][c]{\linewidth}{%
			\centering
			\includegraphics[width=\linewidth]{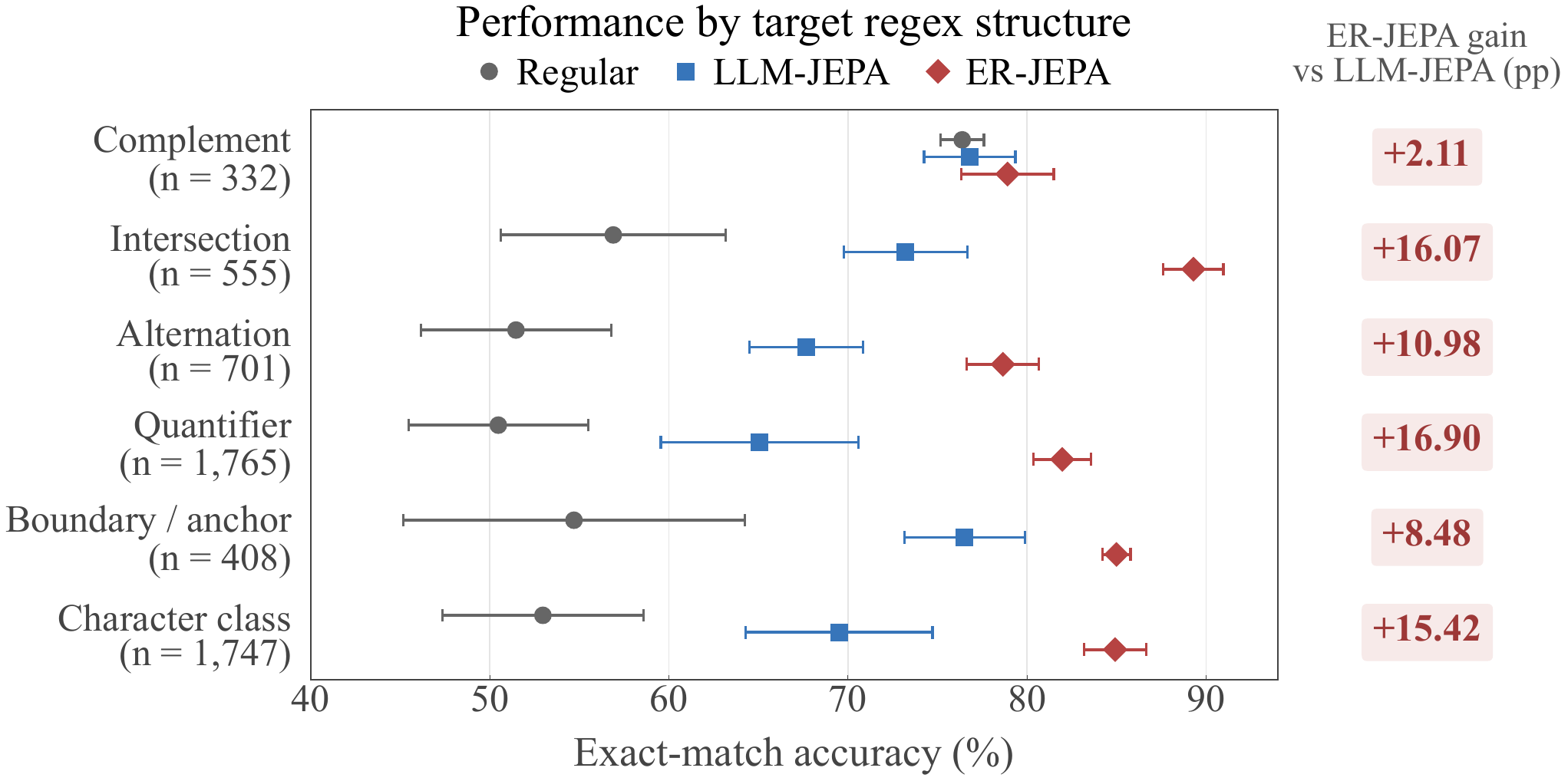}%
		}
		\caption{\textbf{ER-JEPA achieves higher prediction accuracy across all six evaluated regex structures with episodic replay.} Compared with LLM-JEPA, ER-JEPA improves accuracy by $16.90$, $16.07$, and $15.42$ percentage points for expressions containing quantifiers, intersections, and character classes, respectively.}
		\label{fig:regex_structure_accuracy}
	\end{minipage}\hfill
	\begin{minipage}[t]{0.48\textwidth}
		\vspace{0pt}
		\centering
		\parbox[c][0.60\linewidth][c]{\linewidth}{%
			\centering
			\includegraphics[width=\linewidth]{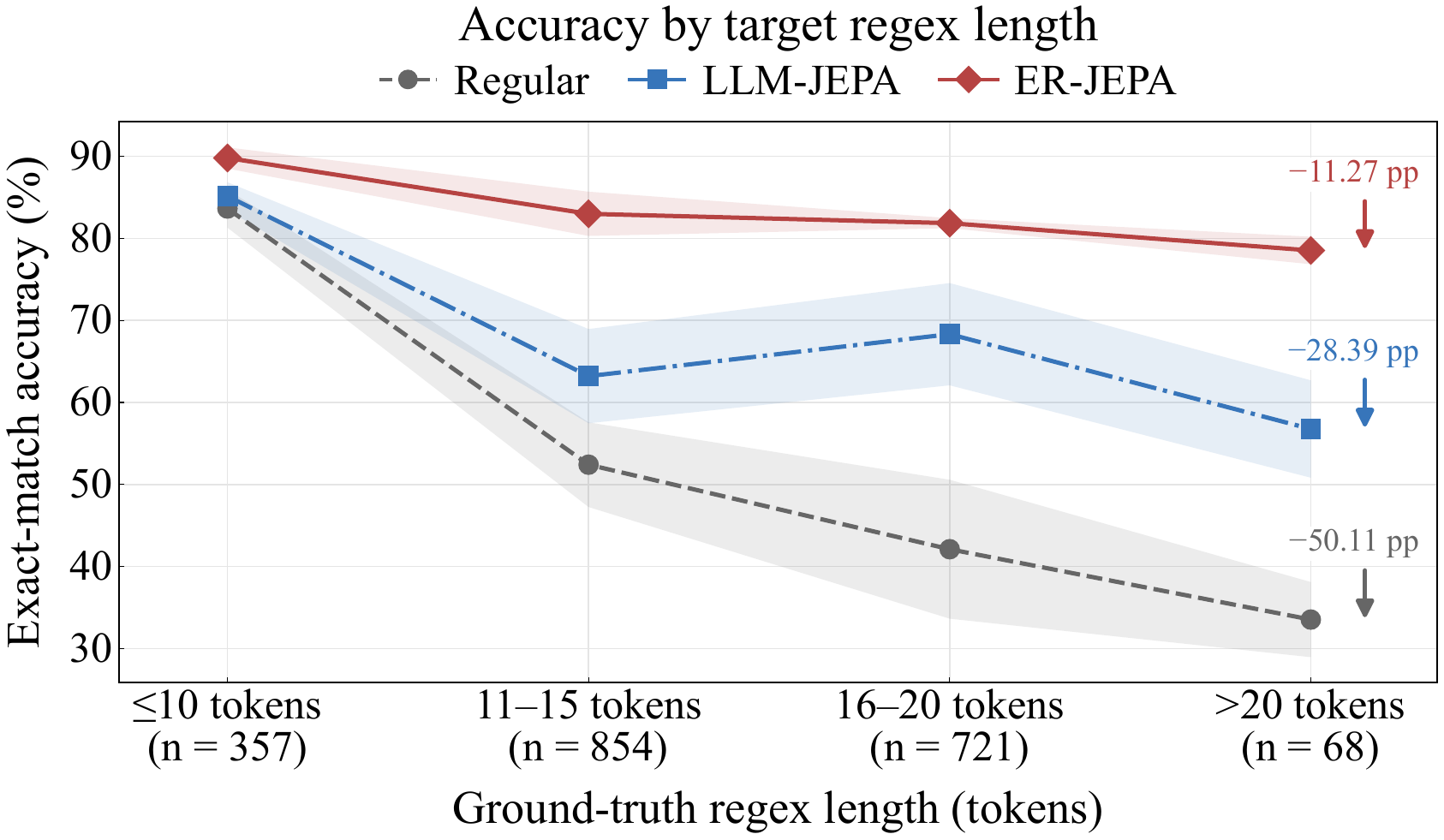}%
		}
		\caption{\textbf{ER-JEPA maintains higher accuracy on longer regular expressions with episodic replay.} ER-JEPA leads in all four length groups and shows a smaller decline from the shortest to the longest group than Regular and LLM-JEPA. Lines show mean exact match accuracy over five seeds, and shaded bands indicate one standard deviation.}
		\label{fig:regex_length_accuracy}
	\end{minipage}
\end{figure*}

\paragraph{Training gains come without additional inference cost.} At inference, ER-JEPA uses the same generation procedure as LLM-JEPA. We remove the episodic memory and replay path after training. Following LLM-JEPA, we use only the fine-tuned language model for inference. Replay therefore adds no overhead to inference time and requires no additional computation or memory during inference.

Ablation studies on memory addressing, memory capacity and JEPA objective hyperparameters are provided in Appendix~\ref{app:ablation}.

\section{Related Work}

Joint-Embedding Predictive Architectures (JEPAs) learn representations by predicting target embeddings from context. \citet{lecun2022path} proposed JEPA as a framework for learning world models that predict abstract states without reconstructing every input detail. A related approach, data2vec~\citep{baevski2022data2vec}, predicts contextualized representations from masked inputs using the same learning method across speech, vision, and language.

In vision, I-JEPA~\citep{assran2023self} predicts the representations of target image regions from a context region, learning semantic features without relying on hand-crafted augmentations. V-JEPA~\citep{bardes2024revisiting} extends latent prediction to video and learns representations that support downstream tasks with a frozen encoder. Image World Models~\citep{garrido2024learning} broaden the prediction task beyond masked regions to include the effects of photometric transformations. To better understand latent prediction, \citet{littwin2024jepa} analyze its learning dynamics in deep linear models. Their analysis identifies an implicit bias toward features with high regression coefficients.

Recent work also applies JEPA to language and vision-language learning. VL-JEPA~\citep{chen2026vljepa} predicts continuous text embeddings for vision-language tasks and uses a decoder when text output is needed. For language models, LLM-JEPA~\citep{huang2026llmjepa} combines next-token prediction with a JEPA loss over paired textual views. This joint objective encourages semantic structure in the learned representations while preserving generative capabilities. Our work builds on LLM-JEPA by adding episodic replay, which reuses past training examples to provide additional supervision for token prediction and representation alignment.

\section{Conclusion}
We propose ER-JEPA, which adds an episodic replay path to LLM-JEPA. The replay path revisits stored training pairs with the current model and provides additional learning signals. At matched training compute, ER-JEPA outperforms LLM-JEPA with all three selection rules. Replay also corrects more errors during training and avoids most large accuracy drops. These results show that revisiting past examples is a simple and effective way to improve LLM-JEPA.

We view the current implementation of ER-JEPA as a first step toward introducing episodic replay into the JEPA framework, and future work could further refine and extend this direction. In particular, since the overall training objective is formulated as a linear combination of the new data JEPA loss and the replay loss, the relative weighting between different loss terms, such as the replay weight $\beta$, currently needs to be selected through grid search. This tuning process introduces additional computational cost and may motivate future work on adaptive weighting strategies. Moreover, ER-JEPA inherits the need of LLM-JEPA for datasets with paired views, such as text and code. We hope future work will explore more efficient and general data augmentation strategies for constructing informative paired views, thereby broadening the application regime of ER-JEPA.


\bibliography{iclr2027_conference}
\bibliographystyle{arxiv-numbered}


\appendix

\suppressfloats[t]
\section{Detailed Experimental Setup}
\label{sec:appendix:experimental_setup}

The experimental design follows the LLM-JEPA protocol. This design allows us to isolate the effect of the proposed episodic replay pathway, rather than introducing confounding changes in data, model scale, or training configuration. All experiments are conducted under the same LoRA fine-tuning setting with rank 512. Under this setting, we compare three fine-tuning objectives: the standard supervised fine-tuning objective, the LLM-JEPA objective, and the proposed ER-JEPA objective.

For each combination of model and dataset, we use a learning rate of $1\mathrm{e}{-5}$ and train for $6$ epochs. We adopt this LoRA configuration because the LLM-JEPA study found that it outperforms fine-tuning all model parameters under the same setting. The standard supervised objective, LLM-JEPA, and ER-JEPA all share this same configuration, learning rate, and training schedule. This helps isolate the effects of the training objective and the episodic-replay extension on performance.

All experiments are repeated with five fixed random seeds,
$\{82, 23, 37, 84, 4\}$.
We report the mean performance across seeds. Statistical significance is assessed using paired one-tailed $t$-tests over seed-matched runs. All training runs are conducted on a server equipped with four NVIDIA A800 GPUs.

\section{How Historical Replay Can Provide Additional Correction Directions}
\label{app:replay_supervision_analysis}

LLM-JEPA combines token generation with representation alignment, but both losses directly supervise only the current batch. We give a conditional explanation of how replay changes one gradient-descent step. Replay can supply a descent direction for the selected historical loss that current-batch gradients do not provide.

\paragraph{The local limitation of the LLM-JEPA objective.}
Let $\omega$ collect the trainable LoRA coordinates, counting shared parameters once. For a nonempty current batch $\mathcal B$, let $\mathcal T_{\mathcal B}$ contain its supervised-token positions and $Z_{\mathcal B}=|\mathcal T_{\mathcal B}|>0$. Write $\pi_q=\operatorname{softmax}(f_q(\omega))$ for the token probabilities, $y_q$ for the target token, and $e_i=p_i/\|p_i\|_2-z_i^a/\|z_i^a\|_2$ for the alignment residual. Equation~\eqref{eq:replay_first:new_loss} gives the LLM-JEPA objective
\begin{equation}
	L_{\mathcal B}(\omega)
	=\underbrace{\frac{\gamma}{Z_{\mathcal B}}\sum_{q\in\mathcal T_{\mathcal B}}-\log\pi_{q,y_q}}_{\text{token generation}}
	+\underbrace{\frac{\lambda}{2|\mathcal B|}\sum_{i\in\mathcal B}\|e_i\|_2^2}_{\text{representation alignment}},
	\label{eq:replay_analysis:loss}
\end{equation}
where $\gamma,\lambda>0$ are fixed. Fix a parameter state $\omega_0$. Hold the examples, sequences, masks, and forward-pass randomness fixed. Repeated occurrences use identical views and forward realizations. For all current and replay examples, assume finite logits, twice continuously differentiable maps, and representation norms above the normalization floor near $\omega_0$.

At $\omega_0$, let $J_q=\partial f_q/\partial\omega$ and $A_i=\partial e_i/\partial\omega$, differentiating both representation branches. The two losses supply the gradient
\begin{equation}
	g_{\mathcal B}=\nabla L_{\mathcal B}(\omega_0)
	=\frac{\gamma}{Z_{\mathcal B}}\sum_{q\in\mathcal T_{\mathcal B}}J_q^\top(\pi_q-\mathbf e_{y_q})
	+\frac{\lambda}{|\mathcal B|}\sum_{i\in\mathcal B}A_i^\top e_i,
	\label{eq:replay_analysis:gradient}
\end{equation}
where $\mathbf e_{y_q}$ is the target-token basis vector. Neither term provides a component along directions in
\begin{equation}
	\mathcal K=\{v:\ J_qv\in\operatorname{span}\{\mathbf1\}\ \forall q\in\mathcal T_{\mathcal B},
	\quad A_iv=0\ \forall i\in\mathcal B\}.
	\label{eq:replay_analysis:blind_directions}
\end{equation}
These directions preserve current token probabilities and alignment residuals to first order. They may still change losses on historical examples. Let $P$ be the orthogonal projector onto $\mathcal K$. The subspace $\mathcal K$ may contain only the zero vector.

ER-JEPA adds $\beta L_{\mathcal R}$, with fixed $\beta>0$, where $L_{\mathcal R}$ applies the same joint objective to a nonempty replay multiset with supervised tokens, as in Eq.~\eqref{eq:replay_first:token_replay_objective}. The analysis also allows replay to overlap $\mathcal B$. Occurrences are counted separately and representations are recomputed. Hold the selected multiset $\mathcal R$ fixed throughout the comparison. Write $g_{\mathcal R}=\nabla L_{\mathcal R}(\omega_0)$.

\begin{replayanalysisprop}[Conditional local effect of historical replay]
	\label{prop:replay_analysis:correction}
	At the fixed state above, $Pg_{\mathcal B}=0$. The same holds for a gradient obtained by repeating current examples or reweighting them with fixed weights. Define $d=-Pg_{\mathcal R}$. If $Pg_{\mathcal R}\ne0$, then
	\begin{equation}
		g_{\mathcal B}^{\top}d=0,
		\qquad g_{\mathcal R}^{\top}d=-\|Pg_{\mathcal R}\|_2^2<0.
		\label{eq:replay_analysis:direction}
	\end{equation}
	For the gradient-descent updates $\omega_{\mathrm C}=\omega_0-\eta g_{\mathcal B}$ and $\omega_{\mathrm E}=\omega_0-\eta(g_{\mathcal B}+\beta g_{\mathcal R})$, we have
	\begin{equation}
		P(\omega_{\mathrm E}-\omega_0)=\eta\beta d,
		\qquad P(\omega_{\mathrm C}-\omega_0)=0.
		\label{eq:replay_analysis:projected_update}
	\end{equation}
	Their replay losses satisfy
	\begin{equation}
		L_{\mathcal R}(\omega_{\mathrm E})-L_{\mathcal R}(\omega_{\mathrm C})
		=-\eta\beta\left(\|Pg_{\mathcal R}\|_2^2+\|(I-P)g_{\mathcal R}\|_2^2\right)+O(\eta^2)
		\label{eq:replay_analysis:benefit}
	\end{equation}
	as $\eta\to0$. If $g_{\mathcal R}\ne0$, then $L_{\mathcal R}(\omega_{\mathrm E})<L_{\mathcal R}(\omega_{\mathrm C})$ for all sufficiently small $\eta>0$. This loss comparison does not require $Pg_{\mathcal R}\ne0$.
\end{replayanalysisprop}

\begin{proof}\renewcommand{\qedsymbol}{}
	Let $v\in\mathcal K$. For each supervised token, write $J_qv=c_q\mathbf1$. Then
	\[
	(J_qv)^\top(\pi_q-\mathbf e_{y_q})=c_q(1-1)=0,
	\qquad (A_iv)^\top e_i=0.
	\]
	Thus every token and alignment contribution is orthogonal to $\mathcal K$. Equation~\eqref{eq:replay_analysis:gradient} gives $Pg_{\mathcal B}=0$. Repeating these contributions or changing their fixed weights preserves this property.
	
	Since $P=P^\top=P^2$, we have $g_{\mathcal B}^\top d=0$ and $g_{\mathcal R}^\top d=-g_{\mathcal R}^\top Pg_{\mathcal R}=-\|Pg_{\mathcal R}\|_2^2$. This proves Eq.~\eqref{eq:replay_analysis:direction} when $Pg_{\mathcal R}\ne0$. Applying $P$ to the two updates gives Eq.~\eqref{eq:replay_analysis:projected_update}.
	
	Taylor expansion at $\omega_0$ gives
	\[
	\begin{aligned}
		L_{\mathcal R}(\omega_{\mathrm C})
		&=L_{\mathcal R}(\omega_0)-\eta g_{\mathcal R}^\top g_{\mathcal B}+O(\eta^2),\\
		L_{\mathcal R}(\omega_{\mathrm E})
		&=L_{\mathcal R}(\omega_0)-\eta g_{\mathcal R}^\top g_{\mathcal B}
		-\eta\beta\|g_{\mathcal R}\|_2^2+O(\eta^2).
	\end{aligned}
	\]
	Subtracting cancels the shared term $-\eta g_{\mathcal R}^\top g_{\mathcal B}$. The orthogonal decomposition $g_{\mathcal R}=Pg_{\mathcal R}+(I-P)g_{\mathcal R}$ then gives Eq.~\eqref{eq:replay_analysis:benefit}. The smoothness assumptions bound the absolute remainder by $C\eta^2$ for some fixed $C>0$ and all sufficiently small $\eta$. If $g_{\mathcal R}\ne0$, choosing $\eta$ small enough that $C\eta<\beta\|g_{\mathcal R}\|_2^2$ makes the difference strictly negative.
\end{proof}

The term $\|Pg_{\mathcal R}\|_2^2$ is the contribution of the additional directions to the first-order loss difference. It is positive only when $Pg_{\mathcal R}\ne0$. The projector is an analytical device; ER-JEPA uses the full replay gradient. The loss comparison also applies to ordinary replay or any added smooth loss with a nonzero gradient. This analysis characterizes a conditional local effect on the replay loss; its implications for predictive performance are evaluated empirically.

\section{Additional Empirical Analyses}
\label{app:additional_empirical_analyses}

\begin{figure*}[t]
	\centering
	\includegraphics[width=0.85\linewidth]{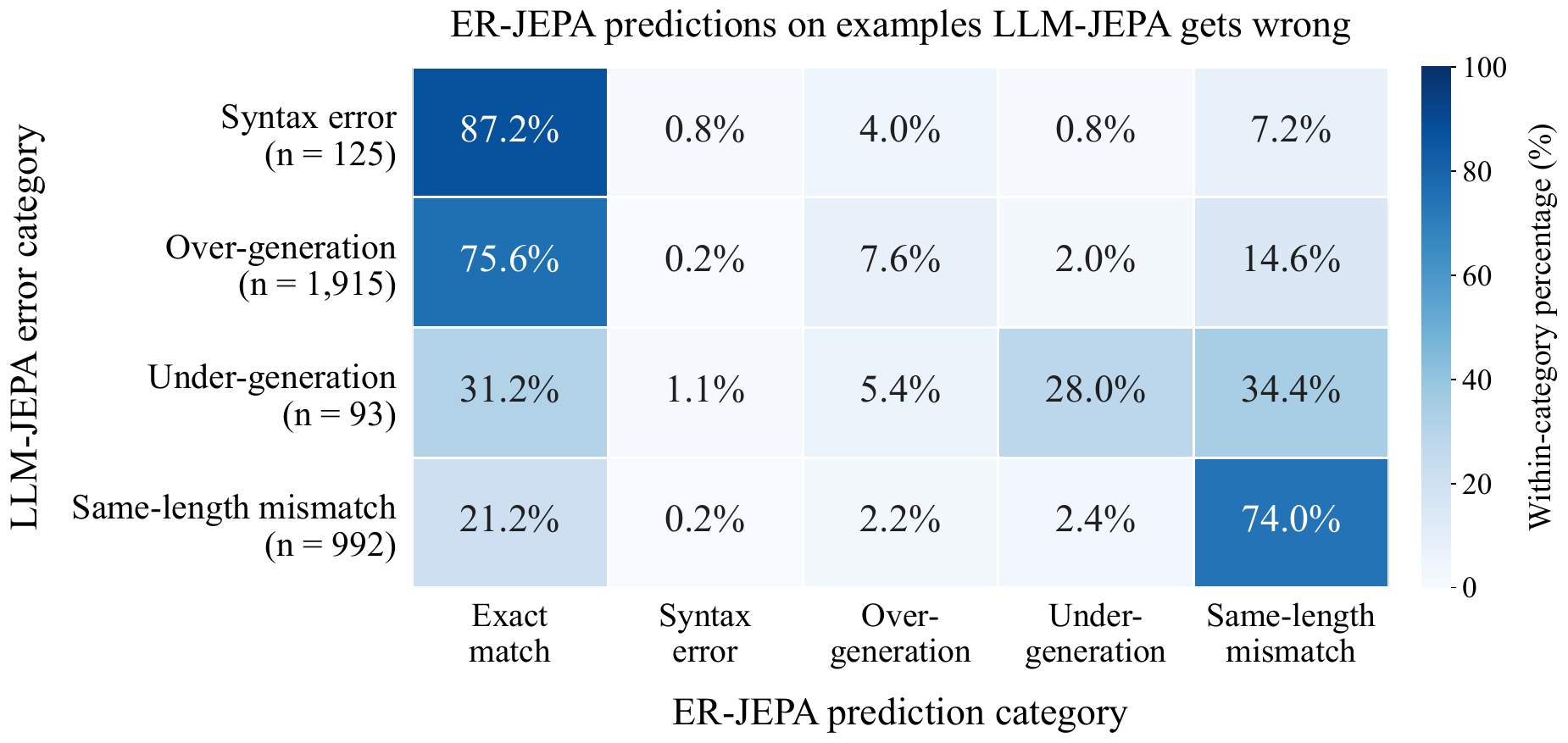}
	\caption{ER-JEPA predictions on examples that LLM-JEPA predicts incorrectly. Rows indicate LLM-JEPA error categories, and columns indicate ER-JEPA prediction categories for the same input and seed. The row total $n$ counts LLM-JEPA errors pooled across five seeds. Each cell shows its count divided by $n$, expressed as a percentage to one decimal place. Results use Llama-3.2-1B on 2,000 SYNTH test examples per seed at a matched training compute budget of 240.309 PFLOPs.}
	\label{fig:llm_jepa_error_outcomes}
\end{figure*}
\subsection{Regex Structural Subsets}
\label{app:regex_structural_subsets}

\begin{table}[t]
	\centering
	\caption{\textbf{Definitions of the six regex structural subsets.} Each subset contains examples whose target regular expressions include the specified feature.}
	\label{tab:regex_structure_definitions}
	\setlength{\tabcolsep}{4pt}
	\renewcommand{\arraystretch}{1.15}
	\begin{tabularx}{\linewidth}{@{}>{\raggedright\arraybackslash}p{0.18\linewidth}>{\raggedright\arraybackslash}p{0.34\linewidth}>{\raggedright\arraybackslash}X@{}}
		\toprule
		\textbf{Subset} & \textbf{Classification rule} & \textbf{Meaning} \\
		\midrule
		Complement
		& Contains the operator \texttt{\textasciitilde}.
		& Excludes strings matched by a subexpression. \\
		\addlinespace
		Intersection
		& Contains the operator \texttt{\&}.
		& Requires multiple patterns to hold simultaneously. \\
		\addlinespace
		Alternation
		& Contains the operator \texttt{|}.
		& Allows a choice between alternative patterns. \\
		\addlinespace
		Quantifier
		& Contains \texttt{*}, \texttt{+}, \texttt{?}, or a brace quantifier such as \texttt{\{m,n\}}.
		& Specifies repetition counts or optionality. \\
		\addlinespace
		Boundary /\newline anchor
		& Contains \texttt{\textbackslash b}, \texttt{\textbackslash B}, \texttt{\textasciicircum}, or \texttt{\$}.
		& Constrains matching positions relative to word boundaries or string endpoints. \\
		\addlinespace
		Character class
		& Contains a bracketed character class, such as \texttt{[a-z]}.
		& Specifies a set of allowed characters. \\
		\bottomrule
	\end{tabularx}
\end{table}

Classification follows the regex grammar. Escaped literal symbols and symbols inside character classes are not counted as separate operators. Each example is assigned to every applicable subset, so the subsets may overlap.

\subsection{Error Correction Across Seeds}
\label{app:baseline_error_correction}

Table~\ref{tab:baseline_error_correction_by_seed} reports the individual results summarized in Table~\ref{tab:baseline_error_correction}. For each seed, we identify Baseline errors among the same 2,000 SYNTH test inputs and evaluate both methods on those inputs. Syntax errors are excluded from the three length-based categories. A prediction counts as corrected only when it exactly matches the target. Each correction rate uses the Baseline error count in its row as the denominator. ER-JEPA uses content replay.

\begin{table}[t]
	\centering
	\caption{\textbf{Correction counts and rates for each seed on the same Baseline errors.} Both methods use the same inputs and seed within each row. Corrected subsets may overlap between methods. The counts and rates in Table~\ref{tab:baseline_error_correction} are summarized across these five seeds.}
	\label{tab:baseline_error_correction_by_seed}
	\setlength{\tabcolsep}{3pt}
	\renewcommand{\arraystretch}{1.1}
	\begin{tabular*}{\linewidth}{@{\extracolsep{\fill}}cccccc@{}}
		\toprule
		Seed & \shortstack{Baseline\\errors} & \shortstack{LLM-JEPA\\corrected} & \shortstack{Correction\\rate (\%)} & \shortstack{ER-JEPA\\corrected} & \shortstack{Correction\\rate (\%)} \\
		\midrule
		\multicolumn{6}{l}{\textbf{(a) Over-generation}} \\
		4 & 901 & 448 & 49.72 & 807 & 89.57 \\
		23 & 697 & 435 & 62.41 & 583 & 83.64 \\
		37 & 705 & 375 & 53.19 & 608 & 86.24 \\
		82 & 591 & 344 & 58.21 & 455 & 76.99 \\
		84 & 698 & 292 & 41.83 & 601 & 86.10 \\
		\addlinespace
		\multicolumn{6}{l}{\textbf{(b) Under-generation}} \\
		4 & 1 & 0 & 0.00 & 1 & 100.00 \\
		23 & 5 & 1 & 20.00 & 2 & 40.00 \\
		37 & 3 & 0 & 0.00 & 1 & 33.33 \\
		82 & 2 & 1 & 50.00 & 2 & 100.00 \\
		84 & 2 & 1 & 50.00 & 0 & 0.00 \\
		\addlinespace
		\multicolumn{6}{l}{\textbf{(c) Same-length mismatch}} \\
		4 & 161 & 29 & 18.01 & 36 & 22.36 \\
		23 & 156 & 31 & 19.87 & 38 & 24.36 \\
		37 & 173 & 41 & 23.70 & 46 & 26.59 \\
		82 & 184 & 41 & 22.28 & 53 & 28.80 \\
		84 & 168 & 32 & 19.05 & 39 & 23.21 \\
		\bottomrule
	\end{tabular*}
\end{table}z

\subsection{Examples Selected by a Fixed Rule}
\label{app:rule_selected_examples}

\begin{figure}[htbp]
	\centering
	\includegraphics[width=\linewidth]{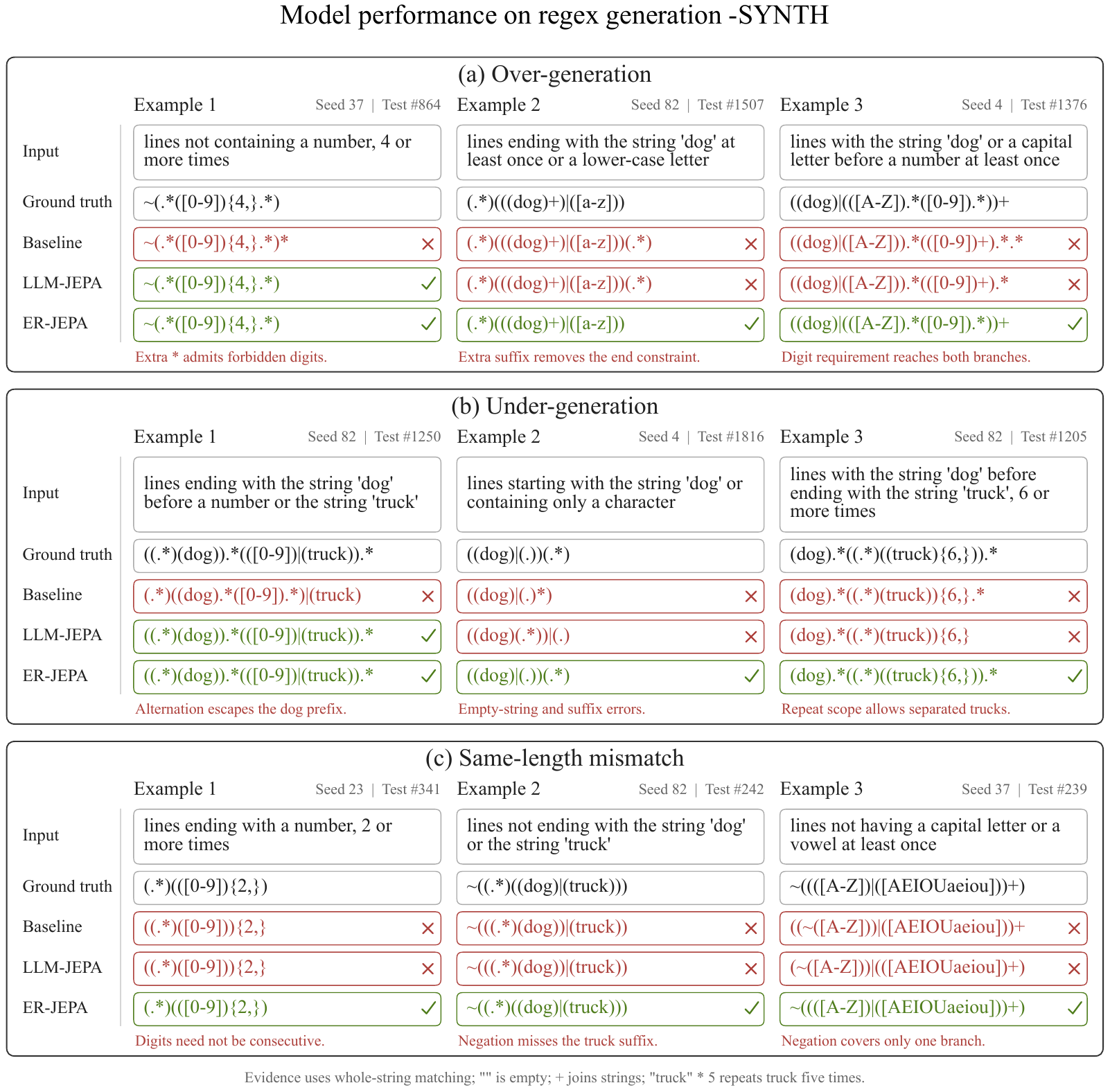}
	\caption{Example predictions from Baseline, LLM-JEPA, and ER-JEPA on the SYNTH dataset.}
	\label{fig:appendix_rule_selected_examples}
\end{figure}

\begin{figure}[htbp]
	\centering
	\includegraphics[width=\linewidth]{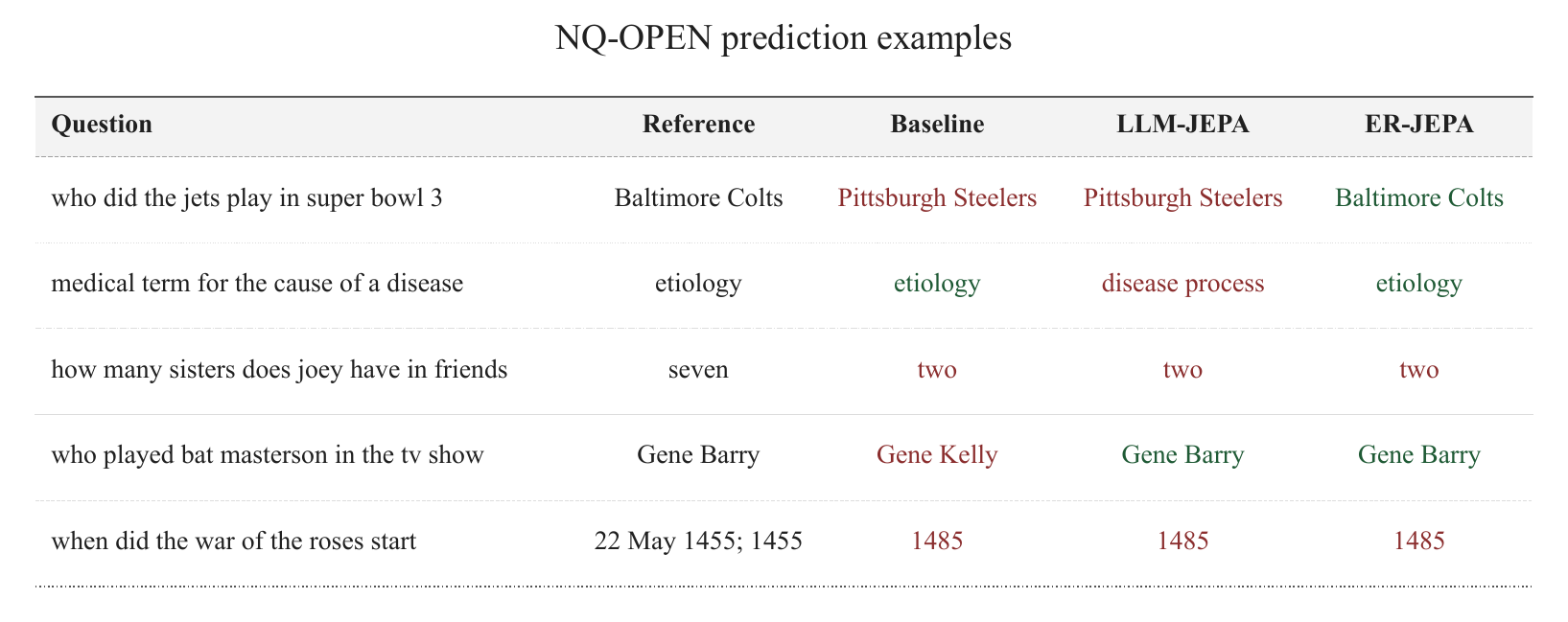}
	\caption{Example predictions from Baseline, LLM-JEPA, and ER-JEPA on the NQ-OPEN dataset. Correct and incorrect answers are shown in dark green and dark red, respectively. For each prediction, only the first semicolon-separated answer is displayed.}
	\label{fig:appendix_nq_open_five_examples}
\end{figure}

Figure~\ref{fig:appendix_rule_selected_examples} illustrates Baseline errors corrected by ER-JEPA, grouped by the Baseline error taxonomy. The selection contains exactly two LLM-JEPA-correct cases and seven LLM-JEPA errors; all three methods are compared on the same input and seed for each example. The full evaluation, using 2,000 test inputs for each of five seeds, gives mean accuracies of $53.63\%$, $68.75\%$, and $83.65\%$ for Baseline, LLM-JEPA, and ER-JEPA, respectively. Figure~\ref{fig:llm_jepa_error_outcomes} reports ER-JEPA predictions on inputs that LLM-JEPA predicts incorrectly.

\subsection{Individual Seed Accuracy at Matched PFLOPs}
\label{app:single_seed_curves}

\begin{figure}[htbp]
	\centering
	\includegraphics[width=0.48\linewidth]{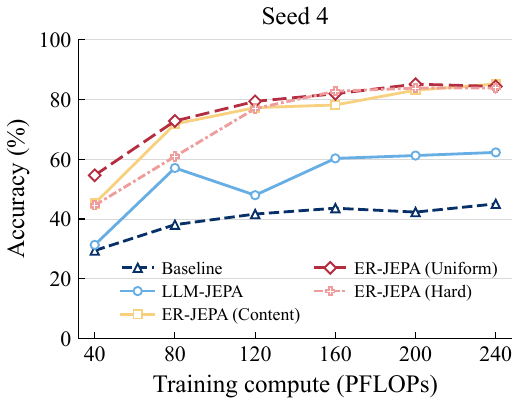}
	\hfill
	\includegraphics[width=0.48\linewidth]{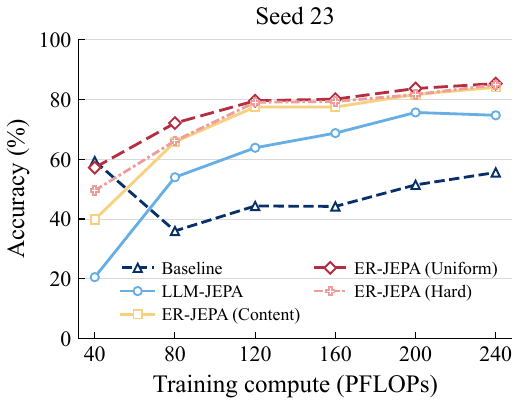}
	
	\includegraphics[width=0.48\linewidth]{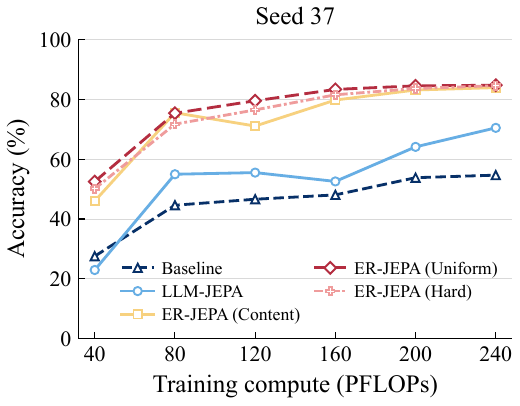}
	\hfill
	\includegraphics[width=0.48\linewidth]{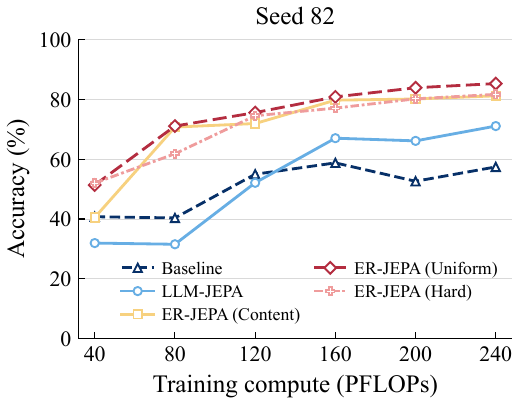}
	
	\caption{Raw accuracy at six shared target PFLOPs for seeds 4, 23, 37, and 82. Each panel shows one seed. Seed 84 is shown in Figure~\ref{fig:accuracy_vs_compute}(b). The panels do not show means or error bars.}
	\label{fig:appendix_single_seed_curves}
\end{figure}

Figure~\ref{fig:appendix_single_seed_curves} complements the seed 84 results in Figure~\ref{fig:accuracy_vs_compute}(b). Across all five seeds, Baseline has a local accuracy decline in four seeds, and LLM-JEPA has a decline in all five. Such declines occur less often with replay. These results suggest that replay mitigates overfitting across individual runs.


\section{ER-JEPA Implementation Details}
\label{sec:appendix:replay_details}

\subsection{Episodic Memory and Retrieval}
\label{sec:method:replay_first:memory}

ER-JEPA maintains a fixed-capacity episodic memory $\mathcal{M}_t$. Each active slot stores a source--target token pair, a sparse address, and memory-management metadata. We denote these entries by
\begin{equation}
	m_n=\big(s_n,\tau_n^u,\tau_n^a,\sigma_n,r_n\big),
	\qquad n\in\mathcal{A}_t,
	\label{eq:replay_first:memory_entry}
\end{equation}
where $\mathcal{A}_t$ is the set of active slots, $s_n$ is the address stored at insertion, $\sigma_n$ is the running JEPA error, and $r_n$ is the replay count. Examples in the current mini-batch are excluded from replay. \textbf{Writing, replacement}, and score \textbf{updates} follow Appendix~\ref{sec:appendix:replay_memory}.

Content-based retrieval uses the current source representation $z_j^u=\mathrm{Enc}_{\theta}(x_j^u)$ as a cue. Let $\Phi$ be the frozen random projection with sparse Top-$K$ selection defined in Appendix~\ref{sec:appendix:replay_content}. We normalize the cue and stored addresses as
\begin{equation}
	\hat{s}_j^{\rm cue}=\frac{\Phi(z_j^u)}{\|\Phi(z_j^u)\|_2},
	\qquad
	\hat{s}_n=\frac{s_n}{\|s_n\|_2}.
	\label{eq:replay_first:normalized_addresses}
\end{equation}
For each cue, we retrieve up to $\kappa$ slots with the highest cosine similarities,
\begin{equation}
	\mathcal{K}_{t,j}
	=\operatorname{Top}_{\min(\kappa,N_t)}
	\left(\left\{\big(n,\langle\hat{s}_j^{\rm cue},\hat{s}_n\rangle\big)
	\mid n\in\mathcal{A}_t\right\}\right),
	\qquad N_t=|\mathcal{A}_t|.
	\label{eq:replay_first:retrieval_selection}
\end{equation}
Here $\operatorname{Top}$ returns slot indices in descending similarity order. We concatenate the lists in batch order, keep the first occurrence of each slot, and retain at most $R$ entries, where $R$ is the replay budget. If fewer than $R$ distinct slots are retrieved, we replay only those slots. Let $\mathcal{S}_t=(n_{t,1},\ldots,n_{t,R_t})$ denote the selected slot list, where $R_t=|\mathcal{S}_t|\leq R$. The token pair for replay entry $i$ is
\begin{equation}
	(\widetilde{\tau}_i^u,\widetilde{\tau}_i^a)
	=(\tau_{n_{t,i}}^u,\tau_{n_{t,i}}^a),
	\qquad i=1,\ldots,R_t.
	\label{eq:replay_first:token_batch}
\end{equation}
The selection rules for all three replay strategies are given in Appendix~\ref{sec:appendix:replay_policies}.

\subsection{Shared Episodic Memory}
\label{sec:appendix:replay_memory}

Each new slot is initialized with
\begin{equation}
	\sigma_n = d(p_n, z_n^a), \qquad r_n = 0,
	\label{eq:replay_first:significance_init}
\end{equation}
where $p_n$ is the predicted target representation and $z_n^a$ is the target readout. When the store is full, the slot with the lowest replay-adjusted significance is evicted,
\begin{equation}
	n^\star = \arg\min_{n\in\mathcal{A}_t} \frac{\sigma_n}{1+r_n}.
	\label{eq:replay_first:eviction}
\end{equation}
After a replay forward pass, the significance of each selected slot is updated as
\begin{equation}
	\sigma_n \leftarrow (1-\eta)\sigma_n + \eta\, d(p_n^{\rm rep}, z_n^{a,{\rm rep}}),
	\qquad
	r_n \leftarrow r_n + 1,
	\qquad \eta\in(0,1),
	\label{eq:replay_first:significance}
\end{equation}
where $\eta$ is the update rate and $p_n^{\rm rep}, z_n^{a,{\rm rep}}$ are recomputed during replay. Unselected slots retain their scores. These rules are shared by all three selection policies.

\subsection{Replay Selection Policies}
\label{sec:appendix:replay_policies}

The content-based, uniform, and hard replay strategies differ only in how replay slots are selected. All use the memory rules in Appendix~\ref{sec:appendix:replay_memory} and the replay loss in Eq.~\eqref{eq:replay_first:token_replay_objective}. Let $\mathcal{A}_t$ be the active slots, $N_t=|\mathcal{A}_t|$, and $R$ the replay budget. A policy $\pi\in\{\mathrm{content},\mathrm{uniform},\mathrm{hard}\}$ returns an ordered slot list $\mathcal{S}_t^\pi$, and $\mathcal{S}_t=\mathcal{S}_t^\pi$. Repeated slots contribute once per occurrence. All policies return an empty selection when $N_t=0$.

\subsubsection{Content Replay}
\label{sec:appendix:replay_content}
\label{sec:method:replay_first:content}

Content replay retrieves stored examples using current source representations $z^u$. Representations are mapped to sparse addresses with a frozen random projection and Top-$K$ selection,
\begin{equation}
	\begin{aligned}
		W_{\rm PS}&\in\mathbb{R}^{S\times H},
		\qquad [W_{\rm PS}]_{\ell k}\sim\mathcal{N}(0,1/H),\\
		[\Phi(z)]_{\ell}
		&=\mathrm{stop\text{-}grad}\!\left(
		\mathbb{I}[\ell\in\mathcal{I}_K(W_{\rm PS}z)]
		(W_{\rm PS}z)_{\ell}\right).
	\end{aligned}
	\label{eq:replay_first:sparse_address}
\end{equation}
$H$ and $S>H$ are the representation and address dimensions, and $\mathcal{I}_K(h)$ contains the indices of the $K$ largest-magnitude entries of $h$, retaining their sign and magnitude. The projection is fixed throughout training.

For current example $j$, the cue is $\hat{s}_j^{\rm cue}=\Phi(z_j^u)/\|\Phi(z_j^u)\|_2$. Let $s_n$ be the address stored at insertion and $\hat{s}_n=s_n/\|s_n\|_2$. The $\min(\kappa,N_t)$ slots with the highest cosine similarity are retrieved,
\begin{equation}
	\mathcal{K}_{t,j}
	=\operatorname{Top}_{\min(\kappa,N_t)}
	\left(\left\{(n,\langle\hat{s}_j^{\rm cue},\hat{s}_n\rangle)
	\mid n\in\mathcal{A}_t\right\}\right),
	\label{eq:replay_first:content_topk}
\end{equation}
returned in descending score order. Lists from all examples in the batch are concatenated in batch order, deduplicated, and truncated to $R$ entries,
\begin{equation}
	\mathcal{C}_t
	=\operatorname{Prefix}_{R}\!\left(
	\operatorname{Unique}\!\left(
	\mathcal{K}_{t,1}\Vert\cdots\Vert\mathcal{K}_{t,|\mathcal{B}_t|}
	\right)\right),
	\qquad
	\mathcal{S}_t^{\rm content}=\mathcal{C}_t,
	\label{eq:replay_first:content_batch}
\end{equation}
where $\operatorname{Unique}$ keeps the first occurrence of each slot and $\operatorname{Prefix}_R$ keeps at most $R$ entries. The default policy uses the resulting list without repetition. For controls with a fixed replay count, we instead use $\mathcal{S}_t^{\rm content}=\operatorname{Fill}_R(\mathcal{C}_t)$, where $\operatorname{Fill}_R$ repeats a shorter nonempty list until it reaches $R$ entries; an empty list remains empty. Cosine similarity is used only for retrieval and contributes no training loss.

\subsubsection{Uniform Replay}
\label{sec:appendix:replay_uniform}
\label{sec:method:replay_first:uniform}

Uniform replay samples stored examples with equal probability, without using cue similarity or significance. When $R\leq N_t$, $R$ slots are sampled without replacement: for any subset $\mathcal{U}\subseteq\mathcal{A}_t$ of size $R$,
\begin{equation}
	\Pr\!\left(\operatorname{set}(\mathcal{S}_t^{\rm uniform})=\mathcal{U}\mid\mathcal{M}_t\right)
	=\binom{N_t}{R}^{-1},
	\qquad
	\Pr(n\in\mathcal{S}_t^{\rm uniform}\mid\mathcal{M}_t)=\frac{R}{N_t}.
	\label{eq:replay_first:uniform}
\end{equation}
When $0<N_t<R$, $R$ slots are drawn independently with replacement, each with probability $1/N_t$.

Uniform replay estimates the mean JEPA error over active memory. Let $d_n(t)$ be the error of stored example $n$ under the current model. For a nonempty memory and fixed parameters,
\begin{equation}
	\mathbb{E}\!\left[
	\frac{1}{R}\sum_{n\in\mathcal{S}_t^{\rm uniform}}d_n(t)
	\;\middle|\;\mathcal{M}_t,\theta
	\right]
	=\frac{1}{N_t}\sum_{n\in\mathcal{A}_t}d_n(t).
	\label{eq:replay_first:uniform_expectation}
\end{equation}

\subsubsection{Hard Replay}
\label{sec:appendix:replay_hard}
\label{sec:method:replay_first:hard}

Hard replay selects the $\min(R,N_t)$ slots with the largest significance $\sigma_n$ (Appendix~\ref{sec:appendix:replay_memory}),
\begin{equation}
	\mathcal{H}_t
	=\operatorname{Top}_{\min(R,N_t)}
	\left(\left\{(n,\sigma_n)\mid n\in\mathcal{A}_t\right\}\right),
	\qquad \mathcal{S}_t^{\rm hard}=\operatorname{Fill}_{R}(\mathcal{H}_t),
	\label{eq:replay_first:hard}
\end{equation}
returned in descending score order, with $\operatorname{Fill}_R$ as in Appendix~\ref{sec:appendix:replay_content}. A large JEPA error indicates poor representation alignment, not necessarily an incorrect generated answer.


\section{Ablation Studies}
\label{app:ablation}

\begin{table}
	\centering
	\caption{Some regular expressions generated by Llama-3.2-1B-Instruct after fine-tuning with $\mathcal{L}_{\rm LLM}$, $\mathcal{L}_{\rm LLM-JEPA}$, and $\mathcal{L}_{\rm ER-JEPA}$ losses. Color code: \colorbox{lightblue}{wrong}, \colorbox{lightpink}{extra}, \colorbox{lightorange}{missing}.}
	\label{tab:more_examples_three_losses}
	\setlength{\tabcolsep}{5pt}
	\renewcommand{\arraystretch}{1.1}
	\begin{tabular}{@{}ll@{}}
		\toprule
		\textbf{Model / target} & \textbf{Regular expression} \\
		\midrule
		\multicolumn{2}{@{}l@{}}{\textbf{Input:} lines not having the string ``dog'' followed by a number, 3 or more times} \\
		\addlinespace
		Ground truth & \textasciitilde((dog.*[0-9].*)\{3,\}) \\
		$\mathcal{L}_{\rm LLM}$ & \textasciitilde((dog.*[0-9].*)\{3,\}) \\
		$\mathcal{L}_{\rm LLM-JEPA}$ & \textasciitilde((dog.*[0-9].*)\{3,\}) \\
		$\mathcal{L}_{\rm ER-JEPA}$ & \textasciitilde((dog.*[0-9].*)\{3,\}) \\
		\midrule
		\multicolumn{2}{@{}l@{}}{\textbf{Input:} lines containing ending with a vowel, zero or more times} \\
		\addlinespace
		Ground truth & .*(.*)(([AEIOUaeiou])*).* \\
		$\mathcal{L}_{\rm LLM}$ & .*(.*)(([AEIOUaeiou])*).*\extra{.*.*} \\
		$\mathcal{L}_{\rm LLM-JEPA}$ & .*(.*)(([AEIOUaeiou])*).*\extra{.*} \\
		$\mathcal{L}_{\rm ER-JEPA}$ & .*(.*)(([AEIOUaeiou])*).* \\
		\midrule
		\multicolumn{2}{@{}l@{}}{\textbf{Input:} lines with a number or a character before a vowel} \\
		\addlinespace
		Ground truth & (([0-9])\textbar(.)).*([AEIOUaeiou]).* \\
		$\mathcal{L}_{\rm LLM}$ & (([0-9])\textbar(.)).*([AEIOUaeiou]).*\extra{.*} \\
		$\mathcal{L}_{\rm LLM-JEPA}$ & (([0-9])\textbar(.)).*([AEIOUaeiou]).* \\
		$\mathcal{L}_{\rm ER-JEPA}$ & (([0-9])\textbar(.)).*([AEIOUaeiou]).* \\
		\midrule
		\multicolumn{2}{@{}l@{}}{\textbf{Input:} lines ending with containing the string ``dog'', 7 or more times} \\
		\addlinespace
		Ground truth & ((.*)(.*dog.*))\{7,\} \\
		$\mathcal{L}_{\rm LLM}$ & ((.*)(.*\extra{(}dog.*))\{7,\}\extra{.*)*} \\
		$\mathcal{L}_{\rm LLM-JEPA}$ & ((.*)\extra{(}(.*dog.*))\{7,\}\wrong{)} \\
		$\mathcal{L}_{\rm ER-JEPA}$ & \miss{ }(.*)\extra{(}(.*dog.*){7,})\wrong{)} \\
		\bottomrule
	\end{tabular}
\end{table}
\subsection{Sparse and Dense Memory Addressing}
\label{sec:experiments:dg_design}

We compare content replay using sparse addresses and dense source representations. Sparse addresses are formed by a frozen random projection and Top-$K$ selection. Both variants use cosine similarity, with all other settings fixed. Table~\ref{tab:dg_sparse_retrieval} reports mean accuracies of $86.37\%$ for sparse retrieval and $86.21\%$ for dense retrieval. The two strategies achieve similar mean accuracy, with a difference of $0.16$ percentage points.

\begin{table}[htbp]
	\centering
	\caption{Sparse and dense memory addressing for content replay. All other training settings are fixed. Results report mean accuracy, standard deviation, minimum, and maximum over five seeds.}
	\label{tab:dg_sparse_retrieval}
	\begin{tabular}{lccc}
		\toprule
		\textbf{Method}
		& \textbf{Accuracy (\%) $\uparrow$}
		& \textbf{Min}
		& \textbf{Max}\\
		\midrule
		ER-JEPA (Sparse Retrieval)
		& $\mathbf{86.37 \pm 0.35}$
		& $85.75$
		& $86.60$ \\
		Dense Retrieval
		& $86.21 \pm 0.37$
		& $85.85$
		& $\mathbf{86.70}$ \\
		\bottomrule
	\end{tabular}
\end{table}

\subsection{Hyperparameter Ablations}

\paragraph{Memory capacity.}
Table~\ref{tab:ablation_memory_capacity} reports accuracy after four epochs, and Figure~\ref{fig:appendix_memory_capacity_curves} shows the learning curves. Increasing memory capacity from $10$ to $10^4$ raises accuracy from $84.22\%$ to $84.98\%$.

\begin{table}[htbp]
	\centering
	\caption{Ablation on episodic memory capacity $\mathcal{M}$ for Meta-Llama-3.2-1B-Instruct on NL-RX-SYNTH. Values at epoch 4 are mean $\pm$ standard deviation over five seeds. The observed training costs are not compute matched.}
	\label{tab:ablation_memory_capacity}
	\setlength{\tabcolsep}{4pt}
	\begin{tabular}{@{}rcccc@{}}
		\toprule
		\shortstack{\textbf{Memory}\\\textbf{capacity $\mathcal{M}$}}
		& \shortstack{\textbf{Accuracy}\\\textbf{(\%) $\uparrow$}}
		& \textbf{PFLOPs}
		& \textbf{Time (min)}
		& \textbf{Tokens (M)} \\
		\midrule
		$10$     & $84.22 \pm 0.75$                  & $181.536 \pm 0.211$ & $10.48 \pm 0.07$ & $8.581 \pm 0.166$ \\
		$10^2$   & $84.32 \pm 0.65$                  & $219.605 \pm 0.533$ & $11.62 \pm 0.22$ & $9.427 \pm 0.175$ \\
		$10^3$   & $84.76 \pm 0.52$                  & $219.766 \pm 0.439$ & $16.74 \pm 1.33$ & $10.969 \pm 0.097$ \\
		$10^4$   & $\mathbf{84.98 \pm 0.35}$         & $219.667 \pm 0.393$ & $22.23 \pm 0.17$ & $11.252 \pm 0.010$ \\
		\bottomrule
	\end{tabular}
\end{table}

\paragraph{Objective hyperparameters.}
We also vary the JEPA loss weight $\lambda$ and predictor depth $k$,
following the LLM-JEPA design.
Figure~\ref{fig:ablation_lambda_k} reports the results.
\clearpage
\begingroup
\makeatletter
\setlength{\@fptop}{0pt}
\setlength{\@fpsep}{12pt}
\setlength{\@fpbot}{0pt plus 1fil}
\makeatother
\begin{figure}[p]
	\centering
	\includegraphics[width=0.7\linewidth]{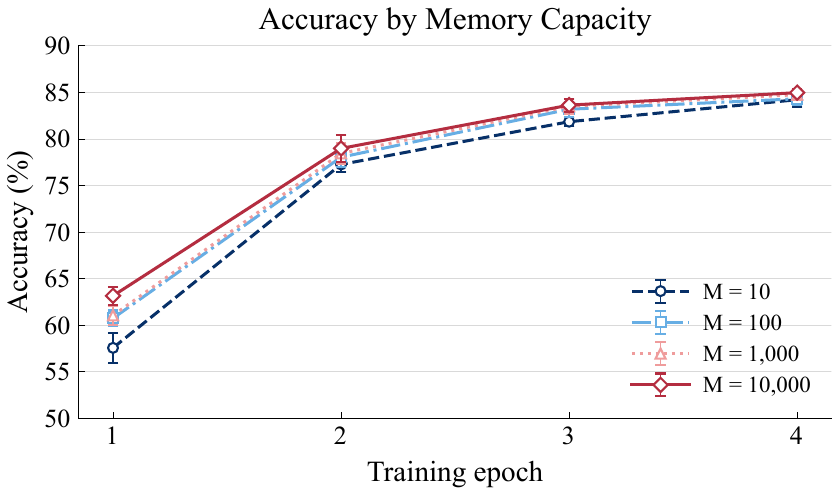}
	\caption{Accuracy across four training epochs for four episodic memory capacities. Points show means over five seeds. Error bars show one sample standard deviation. These runs are not compute matched.}
	\label{fig:appendix_memory_capacity_curves}
\end{figure}

\begin{figure}[p]
	\centering
	\includegraphics[width=0.6\columnwidth]{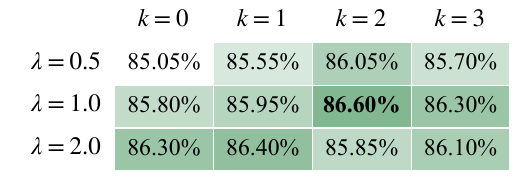}
	\caption{Ablation on the JEPA hyperparameters $\lambda$ and $k$.
		Each entry reports accuracy (\%), with the best result in bold.}
	\label{fig:ablation_lambda_k}
\end{figure}
\clearpage
\endgroup

\end{document}